\documentclass{article} % For LaTeX2e
\usepackage{collas2022_conference,times}

\usepackage[utf8]{inputenc} % allow utf-8 input
\usepackage[T1]{fontenc}    % use 8-bit T1 fonts
\usepackage{hyperref}       % hyperlinks
\usepackage{url}            % simple URL typesetting
\usepackage{booktabs}       % professional-quality tables
\usepackage{amsfonts}       % blackboard math symbols
\usepackage{nicefrac}       % compact symbols for 1/2, etc.
\usepackage{microtype}      % microtypography
\usepackage{xcolor}         % colors

\usepackage{amsmath}
\usepackage{amssymb}
\usepackage{amsthm}
\usepackage{mathtools}
\usepackage{bm}
\usepackage{commath}
\usepackage{wrapfig}
\usepackage{multirow}
\usepackage{tabularx}
\usepackage{cleveref}
\usepackage{hyperref}
\hypersetup{
    colorlinks=true,
    linkcolor=red,
    filecolor=magenta,      
    urlcolor=blue,
    citecolor=purple,
    pdftitle={Few-shot learning by dimensionality reduction in gradient space}
    }

\usepackage{amsmath,amsfonts,bm}

\def\1{\bm{1}}

\DeclareMathAlphabet{\mathsfit}{\encodingdefault}{\sfdefault}{m}{sl}
\SetMathAlphabet{\mathsfit}{bold}{\encodingdefault}{\sfdefault}{bx}{n}

\newcommand{\dE}{\mathbb{E}}

 \newcommand{\dR}{\mathbb{R}}

 \newcommand{\cL}{\mathcal{L}}

\newcommand\rank{\mathrm{rank}}

\DeclarePairedDelimiterX{\kldiv}[2]{(}{)}{%
  #1\;\delimsize\|\;#2%
}

\DeclarePairedDelimiterX{\mi}[2]{(}{)}{%
  #1\;\delimsize ; \;#2%
}

\renewcommand{\leq}{\leqslant}

\newtheorem{lemma}{Lemma}%
\newtheorem*{theorem*}{Theorem}
\newtheorem*{definition*}{Definition}

\newtheorem{theoremA}{Theorem}
\newtheorem{definitionA}{Definition}

\title{Few-Shot Learning by Dimensionality\\ Reduction in Gradient Space}

\author{Martin Gauch,$^{1}$\thanks{correspondence to:\ \texttt{gauch@ml.jku.at}} \,
Maximilian Beck,$^{1}$
Thomas Adler,$^{1}$
Dmytro Kotsur,$^{2}$ 
Stefan Fiel,$^{2}$ \\
\textbf{Hamid Eghbal-zadeh,}$^{1}$
\textbf{Johannes Brandstetter,}$^{1}$\thanks{now at Microsoft Research} \,
\textbf{Johannes Kofler,}$^{1}$
\textbf{Markus Holzleitner,}$^{1}$ \\
\textbf{Werner Zellinger,}$^{3}$
\textbf{Daniel Klotz,}$^{1}$
\textbf{Sepp Hochreiter,}$^{1, 4}$ \textbf{and}
\textbf{Sebastian Lehner}$^{1}$ \\ \\
{$^1$~ELLIS Unit Linz and LIT AI Lab, Institute for Machine Learning}, {Johannes Kepler University Linz, Austria}\\
{$^2$~Anyline GmbH}, {Vienna, Austria}\\
{$^3$~Johann Radon Institute for Computational and Applied Mathematics}, {Austrian Academy of Sciences, Linz, Austria}\\
{$^4$~Institute of Advanced Research in Artificial Intelligence (IARAI), Vienna, Austria}\\
}

\newcommand{\hbv}{HB\kern-0.1emV} % needed because latex kerning sucks

\collasfinalcopy % Uncomment for camera-ready version, but NOT for submission.

\begin{document}

\maketitle
% https://www.lifelong-ml.cc/
%%%%%%%%%%%%%%%%%%%%%%%%%%%%%%%%%%%%%%%%%%%%%%%%%%%%%%%%
%  submissions to CoLLAs should aim for 9 pages,       %
% with a maximum of 10 and no minimum number of pages  %
%%%%%%%%%%%%%%%%%%%%%%%%%%%%%%%%%%%%%%%%%%%%%%%%%%%%%%%%

\begin{abstract}
We introduce SubGD, a novel few-shot learning method which is based on the recent finding that stochastic gradient descent updates tend to live in a low-dimensional parameter subspace. 
In experimental and theoretical analyses, we show that models confined to a suitable predefined subspace generalize well for few-shot learning. 
A suitable subspace fulfills three criteria across the given tasks:\ it 
(a) allows to reduce the training error by gradient flow, 
(b) leads to models that generalize well, and 
(c) can be identified by stochastic gradient descent.
SubGD identifies these subspaces from an eigendecomposition of the auto-correlation matrix of update directions across different tasks.
Demonstrably, we can identify low-dimensional suitable subspaces for few-shot learning
of dynamical systems, which have varying properties described by one or few parameters of the analytical system description. 
Such systems are ubiquitous among real-world applications in science and engineering.
We experimentally corroborate the advantages of SubGD on three distinct dynamical systems problem settings, significantly outperforming popular few-shot learning methods both in terms of sample efficiency and performance.
\end{abstract}

\section{Introduction}
\label{introduction}

Many real-world applications cannot harness the full potential of deep learning when data are limited. 
Examples of such settings range from industrial sensor systems that need adjustments in new conditions to environmental models that must be adapted to a changing climate.
The shortage of data gives rise to the challenge of few-shot learning, i.e., learning from few data samples \citep{hospedales_meta-learning_2020}.

\begin{wrapfigure}[16]{r}{0.4\textwidth}
    \centering
    \vspace{-24.4pt}
    \includegraphics[width=0.9\linewidth]{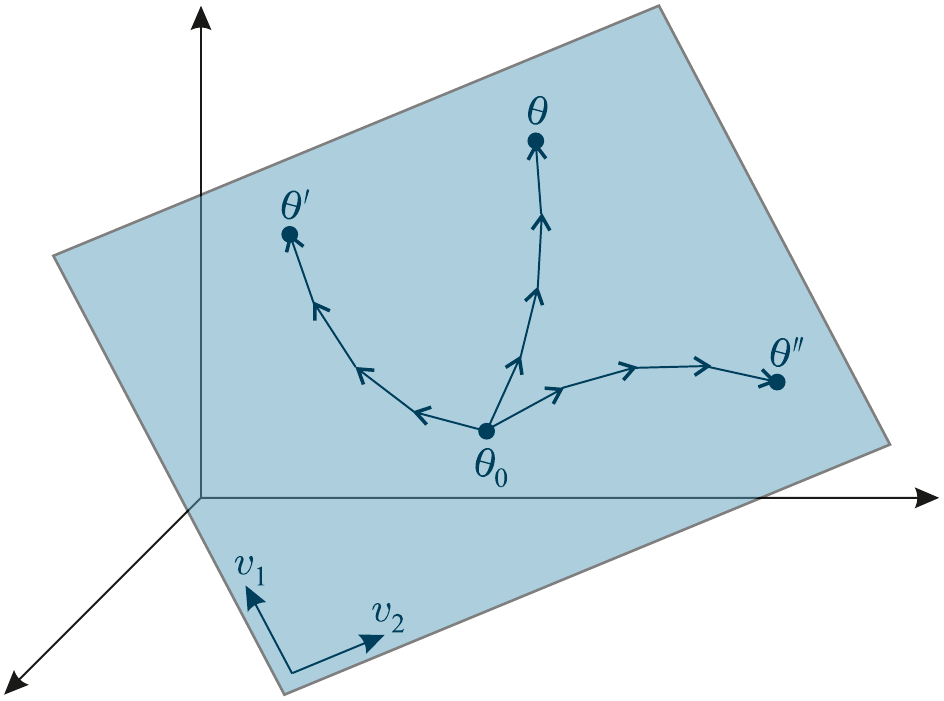}
    \caption{Schematic illustration of SubGD. Starting from an initialization $\theta_0$ in weight space, the update steps (arrows) of three different tasks with end points $\theta$, $\theta'$, and $\theta''$ live in a low-dimensional subspace (spanned by $v_1, v_2$).}
    \label{fig:subspace}
\end{wrapfigure}

We draw inspiration from recent insights into the training dynamics of optimizers based on gradient descent for deep neural networks. In particular, we exploit evidence that after a short burn-in period learning mostly takes place in a remarkably low-dimensional subspace of the parameter space \citep{gur-ari_gradient_2018,advani_high-dimensional_2020}.
Based on these insights, \citet{larsen_how_2021} showed that confining learning to certain low-dimensional subspaces --- even from the beginning of training --- does not deteriorate the performance or even improves it. 
Our work extends these results to the few-shot learning setting, where we demonstrate that stochastic gradient descent (SGD) benefits in terms of sample efficiency if it is restricted to a suitable subspace.

Our method, Subspace Gradient Descent (SubGD), can be outlined as follows.
Starting with a set of weights --- either a random initialization or an existing pre-trained model --- we train the model on each training task individually (Figure~\ref{fig:subspace}).
We identify the aforementioned subspace from an eigendecomposition of the resulting training trajectories. 
To evaluate on new, unseen tasks, we restrict the gradient update steps to the subspace spanned by the eigenvectors with the most dominant eigenvalues and scale the learning rates along these directions by their corresponding eigenvalues.
This can be interpreted as a regularization technique that draws information from the training tasks.
We obtain the learning rate and the number of update steps by fine-tuning on the training tasks or a separate set of hold-out tasks. 
Finally, we fine-tune on unseen test tasks to evaluate the performance. 
Intuitively, restricting the optimization to a subspace that is shared across tasks works well on new tasks, because the optimization is confined to directions that already yielded good generalization on previous tasks.
We further motivate our method theoretically and derive an empirical criterion for the applicability of SubGD (Section~\ref{methods}). 

Most popular few-shot learning benchmarks focus on datasets where different tasks correspond to classification problems that involve new classes of objects \citep[e.g.,][]{ravi_optimization_2017,dumoulin2021unified}.
While the challenge to adapt a model to a new category (i.e., class) in the dataset is highly relevant, we argue that in many practical situations new tasks do not necessarily represent new categories.
Rather, new tasks can stem from changes in the parameters of the data generation processes.
For example, the behavior of an electrical circuit changes as parameters of its components vary, while the underlying causal mechanisms remain unaffected.
This relation between tasks is known as parametric transfer \citep{teshima2020fewshot}.
We find that SubGD is successfully applicable to few-shot learning tasks where the assumption of parametric transfer across tasks appears plausible.
Section~\ref{exps} presents corresponding experimental results. 

The main contribution of this paper is a novel few-shot learning method which we call SubGD.
We show direct connections between SubGD and recent insights into the learning dynamics of SGD.
Moreover, we corroborate the good generalization abilities of SubGD theoretically and experimentally. 
We demonstrate the strengths of our method on a distinct set of few-shot learning problems related to dynamical systems.

\section{Related Work}
\paragraph{Few-Shot Learning}
Many few-shot learning methods rely on meta learning. 
Meta learning can be framed as a nested optimization scheme:\ an inner loop adapts a model to 
a given task, while an outer loop adapts the optimization 
strategy of the inner loop. Meta-learning methods can therefore be 
categorized by how inner and outer loops are implemented. 

Early meta-learning methods use 
\textit{recurrent neural networks} as the inner-loop optimizer 
and some variant of gradient descent on their 
parameters as the meta-learning strategy in the outer loop
\citep{hochreiter_learning_2001,andrychowicz_learning_2016,ravi_optimization_2017}. 
Recurrent networks, however, impose an order on their inputs, 
while learning tasks consist of unordered sets. Therefore, 
also architectures operating on sets are used
\citep{vinyals_matching_2016,ye_few-shot_2020}.

A popular class of few-shot-learning methods is based on 
\textit{fine-tuning}. While some methods fine-tune weights on top of a fixed backbone architecture \citep{adler2020cross,li2021improving,li2021universal}, other methods modify the whole neural network in the inner loop \citep{finn_model-agnostic_2017,nichol_first-order_2018,vuorio_multimodal_2019}. 
Here, the outer loop usually learns an initialization from where 
task optima can be found efficiently. SubGD is 
related to this class of algorithms as it is also based on fine-tuning. However, instead of learning an 
initialization in the outer loop, SubGD modifies the
fine-tuning optimizer. Thus, SubGD is complementary to methods focused on initialization and can naturally extend them.

\textit{Metric-based} methods use geometric relations between samples and employ learners such as nearest-neighbor classifiers as inner-loop optimizers. 
The outer loop adapts the model which 
embeds the samples into the metric space where the learner operates. 
Examples comprise \citet{snell_prototypical_2017}, \citet{vinyals_matching_2016}, \citet{koch_siamese_2015}, and \citet{sung_learning_2018}.

\paragraph{Learning in a Parameter Subspace}
The method put forth in this paper can be viewed as an instance of 
learning in a parameter subspace. 
\citet{raghu_rapid_2020} and \citet{zintgraf_fast_2019} 
constitute subspace-based adaptations of Model-Agnostic Meta Learning \citep[MAML;][]{finn_model-agnostic_2017} that adapt only a subset of the model parameters in the inner loop. 
An algorithmic approach toward the selection of a subspace for fine-tuning is brought forward by \citet{shrinkage}. They determine whether a given group of parameters is to be fine-tuned based on an associated shrinkage parameter which is adjusted during meta-training.
Our proposed method determines the subspace for fine-tuning via an eigendecomposition and requires no manual parameter selection or grouping.

Recent insights into the training dynamics of overparameterized neural networks 
suggest that after a short burn-in period SGD stays within a low-dimensional subspace of the parameter space
\citep{gur-ari_gradient_2018,advani_high-dimensional_2020}. 
\citet{li_measuring_2018} suggest that deep neural networks can be trained in 
random subspaces with far fewer degrees of freedom than the number of parameters, but with comparable performance. \citet{larsen_how_2021} showed that the 
performance can be improved if the subspace is chosen more carefully as the span of the dominant eigenvectors of the gradient descent steps. 
We provide evidence that restricting learning to such subspaces increases the sample efficiency of fine-tuning.  

The complexity of a neural network 
increases with the total number of trainable parameters
\citep{bartlett_vapnik-chervonenkis_2003,koiran_vapnik-chervonenkis_1998}. 
By restricting learning to a low-dimensional subspace, we obtain a model that is 
effectively regularized \citep{advani_high-dimensional_2020}. 
This prevents overfitting on the small number of samples in the few-shot task. 
This regularization technique is informed about the learning dynamics on other similar tasks.
We experimentally demonstrate that this informed regularization is superior to its uninformed version (see Section~\ref{ablations}). 

\paragraph{Learning to Precondition}
Our method can be can be regarded as a few-shot learning approach that learns a preconditioning matrix for the gradient descent update steps on test tasks.  \citet{park_meta-curvature_2019} learn a factorized block-diagonal preconditioning matrix. Effectively, this choice enforces a particular tensor-product structure of the preconditioning matrix which is related to specific model components (e.g., layers and filters). SubGD, in contrast, represents an architecture-independent preconditioning approach.
\citet{flenn_warp_2019} add so-called warp layers at specific parts of the base model. The weights of these warp layers are learned and act like ``slow'' weights that define a precondition matrix for the inner-loop optimization of the task-specific ``fast'' weights. 
Once trained, these layers provide a data-dependent metric in parameter space. The warp layers constrain the inter-dependencies among parameters, because the corresponding subspaces are restricted to the parameters of individual layers.

\section{Methods}
\label{methods}

In few-shot learning, gradient-based training of neural networks is subject to extensive estimation uncertainty due to the limited amount of training data.
We compensate for this corruption of the gradients by applying a suitable preconditioning matrix.
Before we formulate our method, we review a crucial property of SGD.
Let $\cL_S(\theta)$ be the empirical risk with respect to a random sample $S$ (e.g., a dataset or mini-batch) as a function of the network parameters $\theta$. 
Further, let $g = \nabla_\theta \cL_S(\theta) \in \dR^n$ be the stochastic gradient and let $C = \operatorname{\dE}[g g^\top]$
be its auto-correlation matrix, where the expectation is taken over the random sample.
If $C$ is invertible, then $\operatorname{\dE}[g^\top C^{-1} g] = n$ holds because
\begin{equation}
\label{eqn:expectation}
    \operatorname{\dE}[g^\top C^{-1} g]
    = \operatorname{\dE}[\operatorname{Tr}(g^\top C^{-1} g)]
    = \operatorname{Tr}(C^{-1} \operatorname{\dE}[g g^\top]) 
    = \operatorname{Tr}(I_n) = n.
\end{equation} 

Given a learning task and an auto-correlation matrix $C$ of full rank, SubGD chooses an update direction $d \in \dR^n$ that enforces the covariance structure given by Equation~\eqref{eqn:expectation}.
As we formally show in Section~\ref{derivation}, for a parameter vector $\theta \in \dR^n$ the resulting SubGD update rule is 
\begin{equation}
    \theta \gets \theta - \eta \, d, \quad \text{where} \quad d =  C g \label{eqn:subgd_update}
\end{equation}
and $\eta > 0$ denotes the learning rate.
Our few-shot learning approach estimates $C$ based on SGD trajectories obtained from fine-tuning on training tasks (see Section~\ref{procedure}). 
At test time, we apply this estimate in update rule~\eqref{eqn:subgd_update}.
Hereby, we assume that the covariance structure (a) is approximately constant in a limited region of the parameter space and (b) transfers across tasks.  

There is evidence that in the context of SGD with overparameterized 
models $C$ has only few non-negligible eigenvalues \citep{xie_diffusion_2021}. 
Further, $C$ is quadratic in the number of parameters and therefore infeasible for large-scale architectures.
Hence, we do a low-rank approximation via the truncated eigendecomposition $\hat C = V \Sigma V^\top$, where $\Sigma$ is the $r \times r$ diagonal matrix that contains the $r$ largest eigenvalues $\sigma_1, \dots, \sigma_r$ of $C$ in descending order, with $r \leq n$. The $n \times r$ matrix $V$ contains the corresponding eigenvectors in its columns. This confines the training trajectory of SubGD to the $r$-dimensional subspace of the most important update directions. The decomposition of $\hat C$ shows that SubGD projects updates into the span of the eigenvectors in $V$ and scales the learning rates along these directions with the corresponding eigenvalues in $\Sigma$.

Note that the auto-correlation matrix $C$ is uncentered. 
Alternatively, we could consider the (centered) covariance matrix, i.e., perform a principal component analysis (PCA) on the update directions. 
However, this would erase the main update direction from the preconditioning matrix.

The number of subspace dimensions $r$ is a hyperparameter of our method. 
For many practical applications, it is bounded by the size of available memory. 
Otherwise, it can be determined via the distribution of the eigenvalues. 
We find that $C$ typically has few dominant eigenvalues and a large number of very small ones. 
The eigendirections with small eigenvalues are practically irrelevant for learning the tasks. 
That is, the rank of $C$ is effectively low although it might be technically full.
\citet{roy_effective_2007} introduced a real-valued extension of the rank that captures this notion with the concept of Shannon entropy. 
The \textit{effective rank} of $C$ is defined as 
\begin{equation}
    \label{eq:erank_def}
    \operatorname{erank}(C) = \exp\!\Bigg(\!-\sum_{i=1}^n p_i \log p_i\Bigg) \quad \text{where} \quad p_i = \frac{\sigma_i}{\sum_{j=1}^n \sigma_j}.
\end{equation}
In our experiments, we use the effective rank to examine the dimensionality of the subspaces identified by SubGD.

\subsection{Derivation of the Update Direction}
\label{derivation}
In this section, we derive the update direction $d$ similarly to how \citet{pascanu_revisiting_2013} motivate natural gradient descent. 
We consider the constrained optimization problem
\begin{equation}
\min_{\Delta \theta} \cL_S(\theta + \Delta \theta) \quad\text{s.t.}\quad \Delta \theta^\top C^{-1} \Delta \theta = n.
\label{eqn:objective}
\end{equation}
We approximate the empirical risk by a first-order Taylor expansion
and construct the Lagrangian
\begin{equation}
    \cL_S(\theta) + \Delta \theta^\top g + \lambda\, (\Delta \theta^\top C^{-1} \Delta \theta - n)
    \label{eqn:lagrangian}
\end{equation}
with Lagrange multiplier $\lambda$. Setting the derivative with respect to $\Delta \theta$ in~\eqref{eqn:lagrangian} to zero and choosing $\lambda$ according to the constraint in Equation~\eqref{eqn:objective} gives
\begin{equation}
    \Delta \theta = -\sqrt{\frac{n}{g^\top C g}}\, C\, g.
    \label{eqn:subgd_direction}
\end{equation} 
Similar to \citet{pascanu_revisiting_2013}, we absorb the scalar factor in Equation~\eqref{eqn:subgd_direction} into the learning rate $\eta$. Consequently we choose $d = Cg$ for update rule~\eqref{eqn:subgd_update}.
\iffalse
This result determines the direction of $d$ which is used in update rule \eqref{eqn:subgd_update}.
\fi
While $g$ fulfills Equation~\eqref{eqn:expectation} only in expectation, $\Delta \theta$ is chosen to satisfy it exactly on every update. 
Note that this preconditioning of the gradient is somewhat inversely related to second-order methods like the Newton method or 
natural gradient descent \citep{amari_natural_1998}.
These methods multiply the gradient with the inverse of a second-order matrix, i.e., the Hessian matrix or the Fisher information matrix, respectively. 
This is because the goal here is different.
Conventional second-order methods optimize the update direction such that all parameters contribute equally toward minimizing the loss. 
SubGD, in contrast, adjusts learning rates along different directions such that learning is encouraged in important ones and inhibited in others.

\subsection{Procedure for Training and Testing}
\label{procedure}
To get a good estimate for $C$, we have to generate training 
trajectories on the training tasks. 
We refer to this step as fine-tuning, as we usually start these trajectories from a pre-trained initialization.
Since few-shot learning settings typically provide copious data for training tasks, we can draw many samples from each task for these fine-tuning procedures.
We found that in practice, it is often helpful to calculate the preconditioning matrix from aggregated rather than individual update steps --- i.e., from the global differences between fine-tuned and pre-trained weights.
When ample tasks are available, this allows for a more robust estimation of the preconditioning matrix.
We examine this effect empirically in Section~\ref{ablations}.

SubGD is not restricted to SGD optimization.
In fact, it can be coupled with any gradient-based optimizer by modifying the update steps proposed by the optimizer.

If validation tasks are available, we use them to determine good choices 
for the learning rate, the number of update steps.
If validation tasks are not available, we determine these hyperparameters on the training tasks. 
Finally, we turn to the test tasks to determine the few-shot performance.
Given a few samples (support set) of a new task, we fine-tune with SubGD and evaluate the performance on new samples (query set) of the same task.

\subsection{A Generalization Bound for SubGD}
\label{bound}

In this section, we show how our approach is underpinned by a generalization bound for neural networks. 
Consider update rule~\eqref{eqn:subgd_update} with $d = \hat Cg$ in the setting specified in Appendix~\ref{sec:generalization_bound_proof}.
Then, there exists a constant $\zeta > 0$ such that the risk $\cL(\theta)$ of models obtained by this update rule is, with probability at least $1-\delta$, bounded by
\begin{align}
\label{eq:complexity_of_est_error}
        \cL(\theta) \leq \cL_S(\theta)
    +\zeta\sqrt{\frac{\mathrm{rank}(\hat C)\left(\norm{\theta-\theta_0}_1+\log(l)\right)+\log(1/\delta)}{m}},
\end{align}
where $\norm{\cdot}_1$ is the $1$-norm, 
the vectors $\theta, \theta_0 \in \mathbb{R}^n$  are parameterizations of the neural network which are possibly learned on different tasks, $l$ is the Lipschitz constant of the loss function, and $m$ is the sample size. 
Note that $\zeta$ does not depend on the parameters of the neural network, nor does it depend on the dimension of the input.
See Appendix~\ref{sec:generalization_bound_proof} for more details and a proof based on arguments summarized and developed in~\cite{long_generalization_2020}.

As outlined in \citet{larsen_how_2021} and \citet{li_measuring_2018}, typical deep learning problems admit to use a low-rank preconditioning matrix without increasing $\cL_S$.
Equation~\eqref{eq:complexity_of_est_error} suggests that such a matrix can improve generalization.
If $\hat C$, however, is chosen such that $\rank(\hat C)$ becomes too small, then $\cL_S(\theta)$ will increase in turn. Thus, whether a reduction of $\rank(\hat C)$ is beneficial depends on the specific learning problem at hand (see Sections~\ref{sec:sinus} and~\ref{sec:limitations}).
Besides the existence of one low-dimensional subspace containing models with low generalization error, it is crucial that we can identify this subspace.
This can be assured if fine-tuning on each task individually always stays within the same subspace.
In practice, we can assess the dimensionality of a given subspace by means of the effective rank of $\hat C$.

In summary, if the approximation error of the models obtainable by our method is not larger than that of standard SGD \citep[as is, for example, suggested by][]{larsen_how_2021}, then Equation~\eqref{eq:complexity_of_est_error} shows that our method can be expected to obtain a small error for test tasks in the few-shot setting.

\section{Experiments}
\label{exps}
The experiments are structured as follows.
First, we examine SubGD on the sinusoid regression problem from \citet{finn_model-agnostic_2017}. Second, we demonstrate the quality of SubGD in two applications:\ estimating output voltages in a non-linear RLC (resistor, inductor, and capacitor) electrical circuit and adapting environmental models to a changing climate.
\subsection{Sinusoid Regression}
\label{sec:sinus}
The controlled nature of the sinusoid dataset makes it well suited for qualitative analyses and ablation studies. We take advantage of this setting to answer the following questions:
\begin{enumerate}
\itemsep0em 
\item How does SubGD benefit from different pre-trained initializations?
\item How does the subspace dimensionality affect the performance of SubGD?
\item How does the scaling of eigenvector directions by corresponding eigenvalues influence the learning behavior?
\item Would diagonal or random preconditioning matrices yield similar or better results?
\end{enumerate}

For the sinusoid experiments, a task is to predict the output of a sine function $f(x) = a\sin(x - b)$ that is parameterized by an amplitude $a$ and a phase $b$ from few samples.
Following \citet{finn_model-agnostic_2017}, we draw $a$ uniformly from the range $[0.1, 5.0]$, $b$ from $[0, \pi]$, and $x$ from $[-5.0, 5.0]$.
All experiments use a feed-forward neural network with two hidden layers of size 40 and ReLU activations.

\subsubsection{Comparison Against Reference Approaches}
We compare against the few-shot learning methods MAML \citep{finn_model-agnostic_2017}, Reptile \citep{nichol_first-order_2018}, MetaSGD \citep{li_metasgd_2017}, MT-Net \citep{lee_gradbased_2018}, and Meta-Curvature \citep{park_meta-curvature_2019}.
MetaSGD, MT-Net, and Meta-Curvature are of particular interest, since these also constitute methods of learning a preconditioning matrix.

\begin{table}[t]
\caption{MSE on the sinusoid dataset for 5-, 10-, and 20-shot regression after converged fine-tuning. SubGD and MAML start their fine-tuning procedure from the same initialization (50000 iterations of MAML training). Bold values indicate methods that are not significantly worse than the best method ($\alpha=0.01$). SubGD significantly outperforms all other methods for smaller support sizes.}
\label{tab:sinusoid-mse}
\begin{center}
\begin{tabularx}{0.65\linewidth}{Xrrrrrr}
\toprule
\multirow{2}{*}{Method} & \multicolumn{2}{c}{5-shot} & \multicolumn{2}{c}{10-shot} & \multicolumn{2}{c}{20-shot} \\
 & mean & median & mean & median & mean & median \\
\midrule
MAML & 0.301 & 0.153 & 0.102 & 0.049 & \bfseries 0.022 & \bfseries 0.013 \\
Reptile & 0.673 & 0.468 & 0.103 & 0.049 & \bfseries 0.013 & \bfseries 0.008 \\
MetaSGD & 0.410 & 0.230 & 0.249 & 0.043 & 0.143 & 0.086 \\
MT-net & 0.203 & 0.083 & 0.069 & 0.024 & \bfseries 0.025 & \bfseries 0.010 \\
Meta-Curvature & 0.989 & 0.406 & 0.678 & 0.219 & 0.118 & 0.060 \\
SubGD & \bfseries 0.065 & \bfseries 0.031 & \bfseries 0.028 & \bfseries 0.016 & 0.023 & 0.017 \\
\bottomrule
\end{tabularx}
\end{center}
\end{table}

Approaches such as MAML and Reptile that learn an initialization are orthogonal to SubGD.
Starting from the initialization, these methods fine-tune with conventional gradient descent.
In contrast, SubGD does not affect the initialization but uses a fine-tuning procedure that is informed by training tasks.
We thus focus our experimental investigation on the final adaptation to a test set and start the fine-tuning on test tasks for both MAML and SubGD with the same pre-trained initialization --- the one learned by MAML after 50000 iterations.
MetaSGD, MT-Net, and Meta-Curvature, on the other hand, learn both an initialization and a preconditioning matrix for the fine-tuning procedure (albeit a more restricted one than SubGD).
For all methods, we hyperparameter-tune the final learning rate and number of update steps to minimize the average mean squared error (MSE) across 100 tasks.
Since prior work commonly reports the performance after very few update steps, we also provide the MSE after a single update step in Appendix~\ref{app:sinus}. 
Our method, however, focuses on the performance after convergence.
We argue that outside of extremely time- or compute-restricted scenarios, this is the more relevant criterion.

For small support sets, SubGD performs best (Table~\ref{tab:sinusoid-mse}, Figure~\ref{fig:sinusoid-baselines-pretrain}~right).\footnote{Unless stated otherwise, all error bars report 95\% confidence intervals calculated across 100 different tasks. We check for significance with two-sided Wilcoxon signed-rank tests.}
With increasing support size, the predictions of MAML, Reptile, and MT-net improve, until their results are similar to SubGD at support size 20.
The comparably worse performance of the remaining baselines appears to stem from their less flexible adaptation at test time.
MAML and SubGD can be fine-tuned for hundreds of update steps at a low learning rate (even when MAML is trained with few inner-loop steps). The results in Table~\ref{tab:sinusoid-mse} show that MetaSGD performs a good first-step update, but does not improve further when fine-tuned until convergence with a low learning rate.

\begin{figure}
    \centering
    \includegraphics[width=\linewidth]{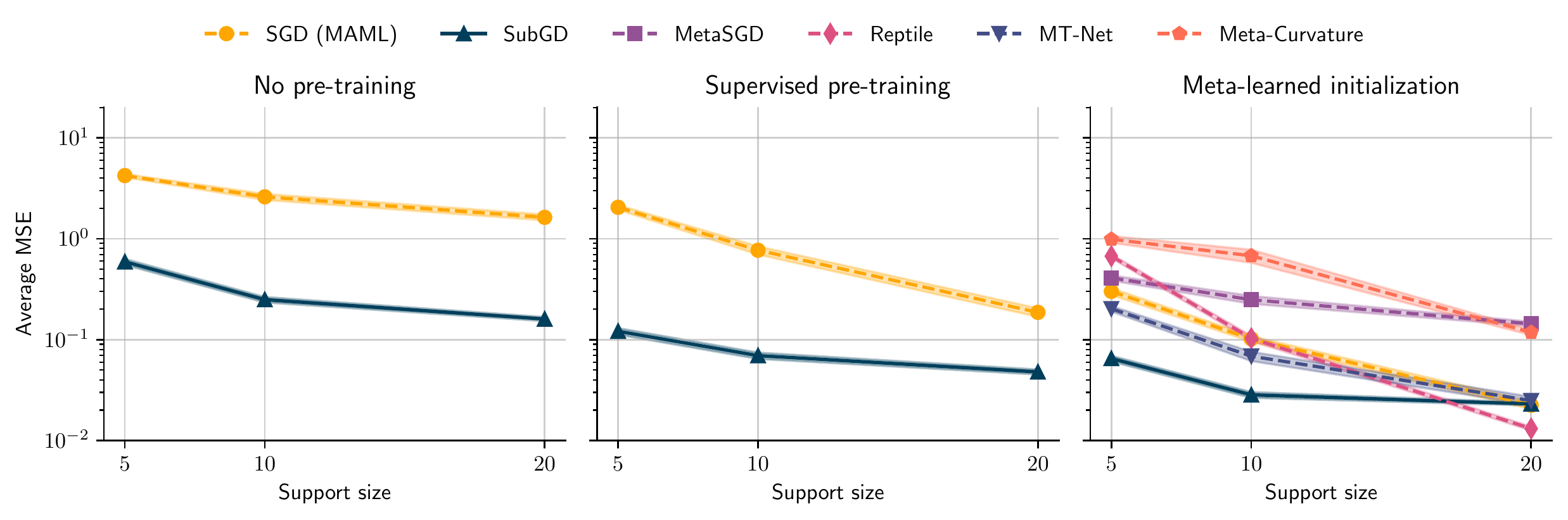}
    \caption{Average MSE with increasing support size on the sinusoid dataset. We show results for SubGD and benchmark methods, starting from different pre-training strategies. No pre-training (left):\ SGD and SubGD start from a random initialization. Supervised (middle):\ fine-tuning from an initialization learned on copious data from a single sinusoid configuration. Meta-learned initialization (right):\ SGD and SubGD fine-tune from a MAML initialization. Hence, in this context, SGD is equivalent to MAML. The remaining baselines learn their solutions from scratch. We tune the fine-tuning learning rate and number of update steps. Regardless of the initialization, SubGD is most sample efficient.}
    \label{fig:sinusoid-baselines-pretrain}
    \vspace{-10pt}
\end{figure}

\subsubsection{Ablations}
\label{ablations}

\begin{wrapfigure}[18]{r}{0.4\textwidth}
    \centering
    \vspace{-20pt}
    \includegraphics[width=\linewidth]{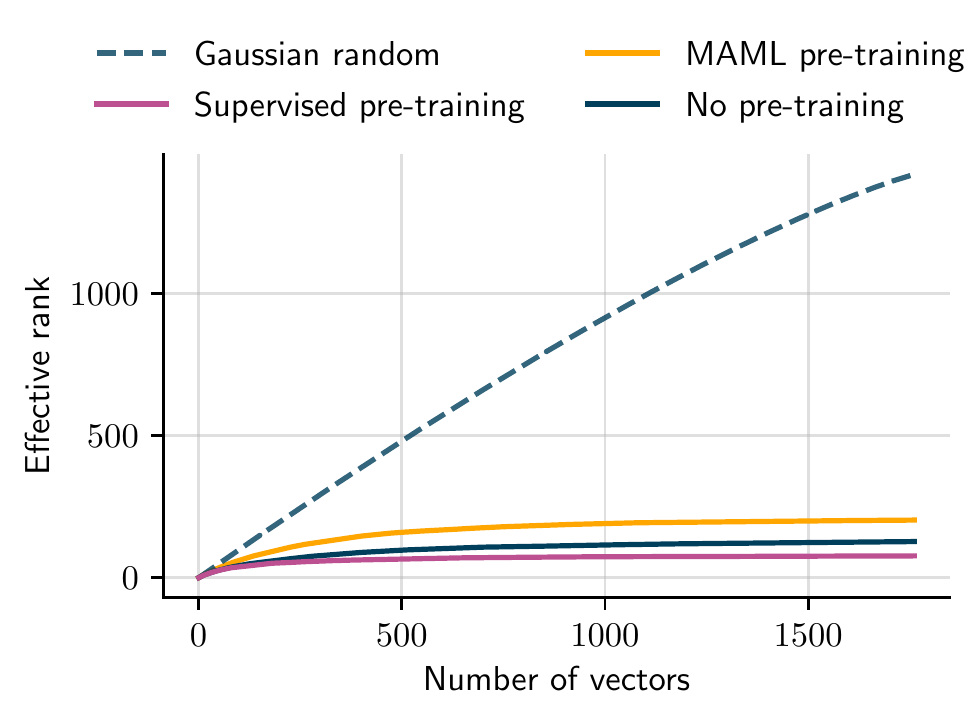}
    \caption{Effective rank of the differences between the pre-trained initialization and up to 1761 fine-tuned model parameters for different initialization strategies (1761 is the number of model parameters). The dashed line shows the effective rank of random Gaussian vectors.}
    \label{fig:sinusoid-erank}
\end{wrapfigure}

\paragraph{Pre-training Strategy}
As noted in Section~\ref{procedure}, SubGD can be combined with methods such as MAML that learn an initialization.
Figure~\ref{fig:sinusoid-baselines-pretrain} illustrates the effect of different pre-training strategies on the performance.
We compare SubGD and SGD after fine-tuning from a random initialization (Figure~\ref{fig:sinusoid-baselines-pretrain}~left), an initialization that is learned on a single sinusoid configuration ($a=2.5$, $b=\pi / 2$; center), and a MAML initialization (right).
For simplicity, we will refer to pre-training on a single configuration as ``supervised pre-training''.
More informed pre-training improves the predictions, and a meta-learned initialization works best.
SubGD is most sample efficient and yields lower errors than SGD-based fine-tuning.
In fact, SubGD is already competitive with few-shot learning baselines when we pre-train on just a single task or skip pre-training entirely.

Figure~\ref{fig:sinusoid-erank} shows that the effective rank saturates after a sufficient number of tasks.
All further tasks can be solved within the subspace spanned by previous tasks.
This observation holds for all pre-training strategies. 
Based on these results, the strong performance of SubGD is expected:\
when few samples are available, the subspace leads to solutions that generalize better to unseen query samples.

\paragraph{Subspace Size and Weighting}
For supervised learning settings outside of few-shot learning, \citet{larsen_how_2021} found that there exists a sweet spot in the subspace dimensionality where stochastic gradient descent yields the best performance.
In our setting, the effective dimensionality is further impacted by the weighting of dimensions by their eigenvalues.
Since many eigenvalues are close to zero, the weighting effectively reduces the dimensionality of the parameter subspace.
In this experiment, we vary the number of subspace dimensions between 2 and 1761 (the total number of model parameters) for SubGD with and without weighting.

For SubGD without weighting, we observe a distinct performance optimum at subspace size 8, which deteriorates as the subspace dimensionality further increases  (Figure~\ref{fig:sinusoid-subspacesize}).
When we additionally weight the subspace directions by their eigenvalues, the prediction quality maintains its high level as subspace size increases up to the full model size.
The distribution of eigenvalues in Figure~\ref{fig:sinusoid-eigenvalues} explains this behavior.
Notably, the eigenvalue size rapidly decreases after the first few dimensions.
Consequently, SubGD fine-tuning will barely move in directions that correspond to low eigenvalues, which constitutes an automatic regularization entirely learned from training information.
For practical applications, this is an important result, as one does not need to tune the subspace size. Instead, we can rely on SubGD to restrict the subspace to the ideal effective subspace dimensionality.

\begin{wrapfigure}[22]{r}{0.65\textwidth}
    \centering
    \vspace{-20pt}
    \includegraphics[width=\linewidth]{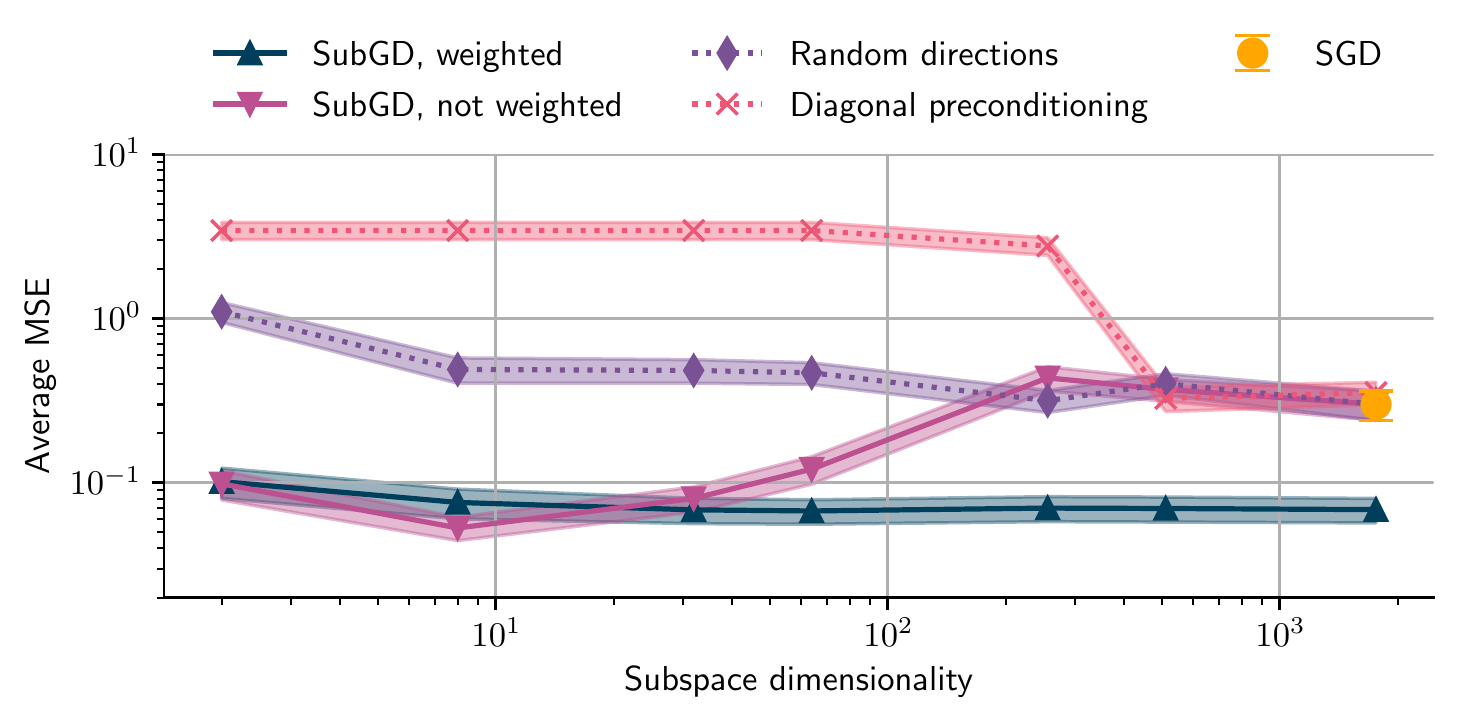}
    \caption{Average MSE across 200 tasks at support size 5 with increasing subspace dimensionality. We compare SubGD with and without weighting by eigenvalues, random subspaces, and diagonal preconditioning based on squared weight differences. The yellow marker indicates the reference performance of fine-tuning without any preconditioning. For SubGD, we added dimensions along the directions of the highest eigenvalues. For diagonal preconditioning, we analogously added the highest diagonal entries.}
    \label{fig:sinusoid-subspacesize}
\end{wrapfigure}

\paragraph{Subspace Complexity}
Methods such as MT-net and MetaSGD use diagonal preconditioning matrices (i.e., per-parameter learning rates), while SubGD leverages more complex subspace structures.
This raises the question whether it is necessary to use the entire SubGD preconditioning matrix, or whether a simpler and computationally cheaper diagonal matrix would suffice.
We thus compare our method against two simpler variants.
First, we directly use the mean squared differences between fine-tuned and pre-trained models as a preconditioning matrix.
In this case, we skip the eigendecomposition of weight differences and use the squared differences as per-parameter learning rates.
Effectively, this results in axis-aligned subspaces, where we subsequently zero out the diagonal entries with the lowest squared differences.
Second, we restrict the gradient descent to random directions --- a procedure motivated by \citet{larsen_how_2021}, who successfully applied this method in a single-task many-shot setting.
Since there might exist a sweet spot in the subspace dimensionality of the different preconditioning variations, we compare the approaches at different numbers of dimensions.
For SubGD, we restrict the dimensionality by subsequently removing directions with the lowest eigenvalues.
Figure~\ref{fig:sinusoid-subspacesize} shows that the non-axis-aligned nature of SubGD subspaces is crucial to the success of fine-tuning.
While SubGD works well at all evaluated subspace dimensionalities, the diagonal and random preconditioning are, at best, as good as SGD fine-tuning that does not leverage information from training tasks beyond the initialization.
These results are in line with the baseline evaluation against MetaSGD, where SubGD yields more accurate predictions.

As noted in Section~\ref{procedure}, there exist multiple options to calculate the preconditioning matrix --- either from global differences between pre-trained and fine-tuned weights, or from more fine-grained update steps (e.g., epoch-wise or individual gradient steps).
In practice, we found the global differences to be more reliable.
With global differences, SubGD achieves an average MSE of 0.067, while SubGD with epoch-wise differences results in a significantly higher MSE of 0.182.
Here, one epoch corresponds to four gradient update steps.
The reason for this effect appears to be related to mini-batch noise.
While epoch-wise steps provide only noisy estimates of the ideal optimization direction, the global differences are more robust to noise from the mini-batches and the training data.
\clearpage

\begin{figure}[h]
\centering
\begin{minipage}[c]{0.42\textwidth}
    \centering
    \vspace{-10pt}
    \includegraphics[width=\linewidth]{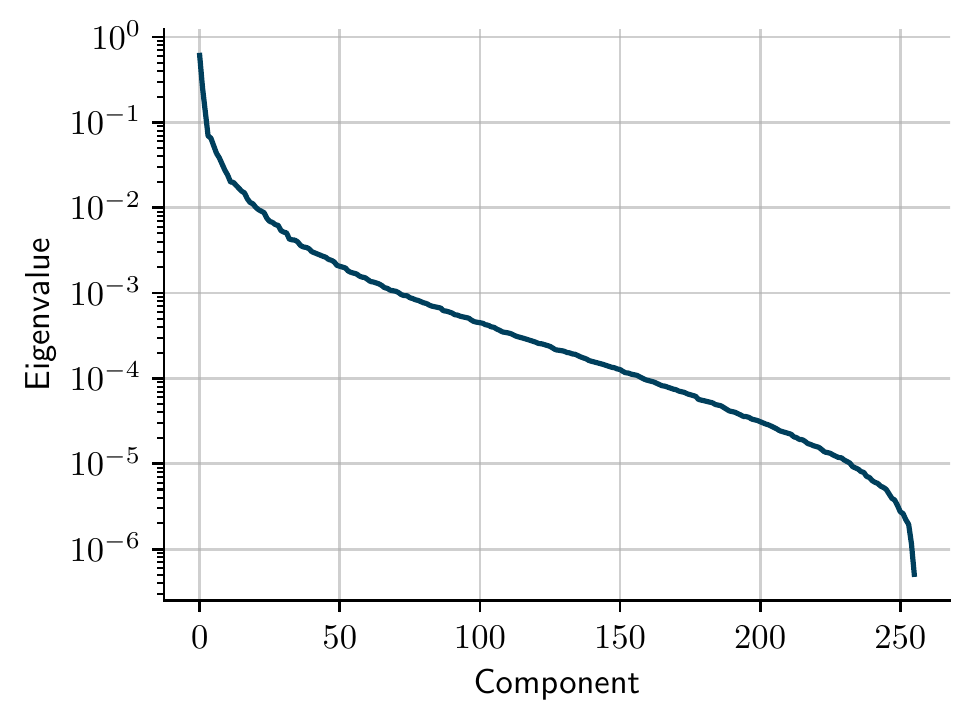}
    \vspace{-1pt}
    \caption{Eigenvalue distribution for a preconditioning matrix generated from 256 different sinusoid regression tasks.}
    \label{fig:sinusoid-eigenvalues}
\end{minipage}\hspace{4em}
\begin{minipage}[c]{0.42\textwidth}
    \centering
    \includegraphics[width=\linewidth]{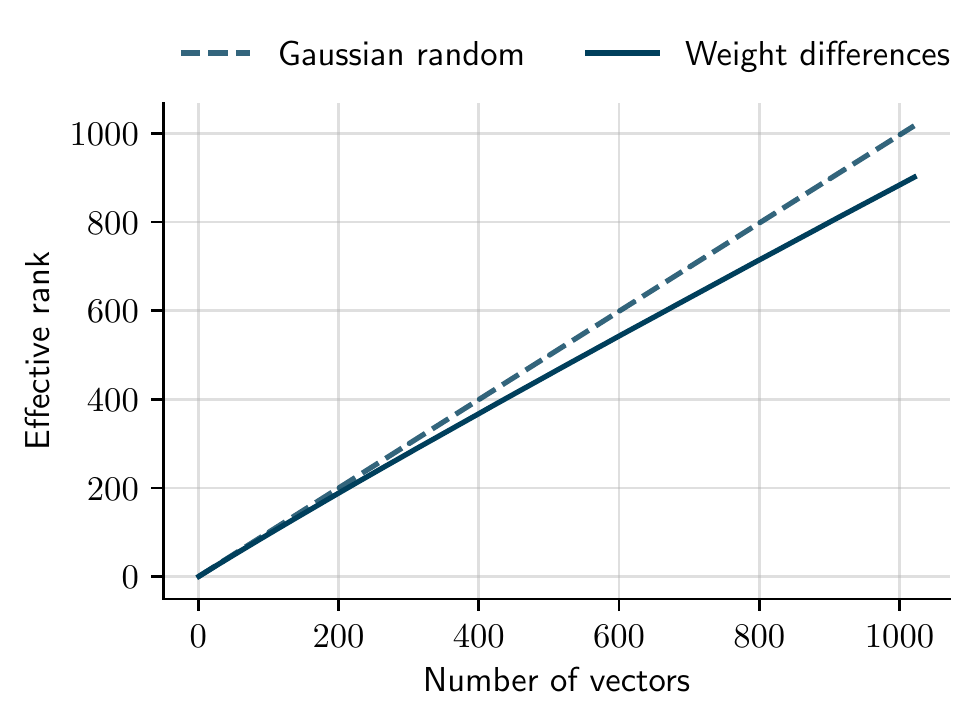}
    \caption{Effective ranks of the vectors that correspond to training trajectories of miniImageNet tasks (solid) and for random Gaussian vectors (dashed). The trained model is a 4-layer Convolutional Neural Network.}
    \label{fig:miniimmage-erank}
\end{minipage}
\end{figure}

\subsection{Limitations}
\label{sec:limitations}

While our experiments on the toy sinusoid dataset and the real-world applications show how well SubGD performs on dynamical systems, we also believe that it is important to outline the limitations of our method.
By design, SubGD only works if we can identify a shared subspace on the training tasks that allows solving unseen tasks at test time.
Empirically, this requirement is often fulfilled, as our experiments show empirically (e.g., Figure~\ref{fig:sinusoid-erank}).
The tasks in these experiments share the dynamics that generate the data. Differences between individual tasks result from the variation of one or few parameters in the analytical description of the dynamical systems.
In contrast, tasks that are commonly encountered in computer vision are not obviously related in this parametric way. 
Here, different tasks typically correspond to entirely new classes of objects. The popular miniImageNet few-shot learning benchmark \citep{vinyals_matching_2016, ravi_optimization_2017}, for instance, samples tasks as classification problems with a varying set of classes.
Indeed, our experiments indicate that the effective rank of weight differences from different miniImageNet tasks (5-way 5-shot classification) does not flatten with increasing numbers of tasks (Figure~\ref{fig:miniimmage-erank}).
In other words, we cannot identify a shared low-dimensional subspace between all tasks and therefore cannot expect SubGD to work in this setting.
This observation is in line with the experimental results on miniImageNet, which indicate that applying SubGD with a 1024-dimensional subspace (based on the same vectors as Figure~\ref{fig:miniimmage-erank}) yields a worse accuracy than normal SGD.

\subsection{Non-Linear RLC}
\label{rlc}
In this experiment, we consider non-linear electric circuits where a resistance $R$, an inductance $L$, and a capacitance $C$ are connected in series.
These circuits are one particular instance of physical systems that exhibit oscillatory behavior.
They are ubiquitous in numerous devices, such as mobile phones, radio, or television receivers.
Specifically, we adapted the setting from \citet{forgione_si_2021}, where the goal is to approximate the dynamics of an RLC circuit with unknown parameters.
We translate this setting to a few-shot learning context.
Given data from prior parameter variations (i.e., training tasks), our goal is to adapt to new variations (i.e., test tasks) with only few data points.
Such adaptations are necessary when the system's operating conditions change, the system is subject to wear and tear, or when the system behavior deviates from the desired nominal behavior due to manufacturing inaccuracies.

Following these considerations, we define a task as one parameter instantiation of the dynamical system and sample the parameters $R$, $L$, and $C$ uniformly. 
We generate our training and test data by simulating the system response to different input signals.
Like \citet{forgione_si_2021}, we model the dynamics of the system with a feed-forward neural network and use the forward-Euler discretization scheme for training. 
Appendix~\ref{sec:appendix:rlc} provides further setup details.

We assess the ability of SubGD to adapt the neural network to new tasks.
We compare SubGD against fine-tuning of a pre-trained model, the few-shot learning methods of first-order MAML \citep[foMAML;][]{finn_model-agnostic_2017}, and Reptile, as well as two further methods that employ a preconditioning matrix (MetaSGD and Meta-Curvature). 
Additionally, we provide Jacobian Feature Regression \citep[JFR;][]{forgione_si_2022} as a baseline. 
JFR is a method for computationally efficient transfer learning through linearization of a neural network \citep{maddox_linearized_2021}.

Figure~\ref{fig:rlc-support-vs-mse} and Table~\ref{tab:rlc-mse} show the median MSE of SubGD and SGD across all test tasks using initializations from supervised pre-training, Reptile and foMAML. 
We observe that SubGD consistently achieves the best performance for all initializations and support sizes.
For larger support sizes, the initialization turned out to be less important, as SubGD reaches almost the same median MSE for all pre-training strategies.
JFR cannot keep up with the other methods and even deteriorates the model performance when applied in the scarce-data regime.
These results lead to the conclusion that SubGD improves generalization to unseen test tasks.

\begin{figure}[t]
    \includegraphics[width=\linewidth]{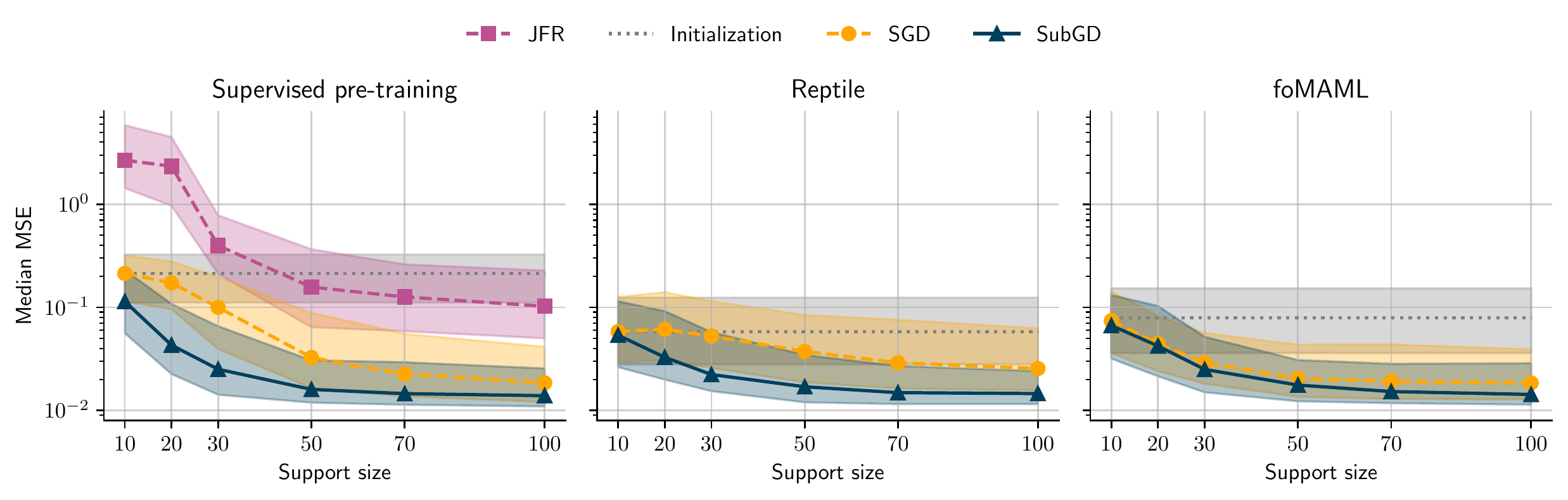}
    \caption{Median MSE with increasing support size on the RLC dataset. We compare SDG and SubGD fine-tuning starting from different initializations, as well as JFR. The performance of the JFR weight initialization is virtually identical to that of supervised pre-training on a single system and hence omitted. We report the median MSE due to outliers. Error bars show the first and third quartile over 256 test tasks. Already with plain supervised pre-training, SubGD is competitive with SGD fine-tuning of more informed initializations.}
    \label{fig:rlc-support-vs-mse}
\end{figure}
\begin{table}[t]
\caption{Median MSE across 256 tasks for increasing support size on the RLC dataset. Bold values indicate methods that are not significantly worse than the best method $(\alpha=0.01)$. For all support sizes, the best model employs SubGD.}
\label{tab:rlc-mse}
\begin{center}
\begin{tabularx}{0.8\linewidth}{Xrrrrrr}
\toprule
\multirow{1}{*}{Method} & \multicolumn{1}{c}{10-shot} & \multicolumn{1}{c}{20-shot} & \multicolumn{1}{c}{30-shot} & \multicolumn{1}{c}{50-shot} & \multicolumn{1}{c}{70-shot} & \multicolumn{1}{c}{100-shot}\\
\midrule
JFR & 2.673 & 2.334 & 0.400 & 0.157 & 0.126 & 0.103  \\
MetaSGD             &0.072 &0.048 &0.035 &0.021 &0.023 &0.022 \\
Meta-Curvature      &0.062 &\bfseries0.038 &0.029 &0.020 &0.018 &0.017 \\
Supervised+SGD      &0.213 &0.173 &0.100 &0.033 &0.023 &0.019 \\
Supervised+SubGD    &0.114 &\bfseries0.043 &\bfseries0.025 &\bfseries0.016 &\bfseries0.015 &\bfseries0.014 \\
foMAML              &0.074 &0.045 &0.029 &0.020 &0.020 &0.019 \\
foMAML+SubSGD       &0.067 &0.042 &\bfseries0.025 &\bfseries0.018 &\bfseries0.015 &\bfseries0.014 \\
Reptile             &0.058 &0.061 &0.053 &0.037 &0.029 &0.026 \\
Reptile+SubGD       &\bfseries0.054 &\bfseries0.033 &\bfseries0.022 &\bfseries0.017 &\bfseries0.015 &\bfseries0.015 \\
\bottomrule
\end{tabularx}
\end{center}
\end{table}

\begin{figure}
    \centering
    \includegraphics[width=\linewidth]{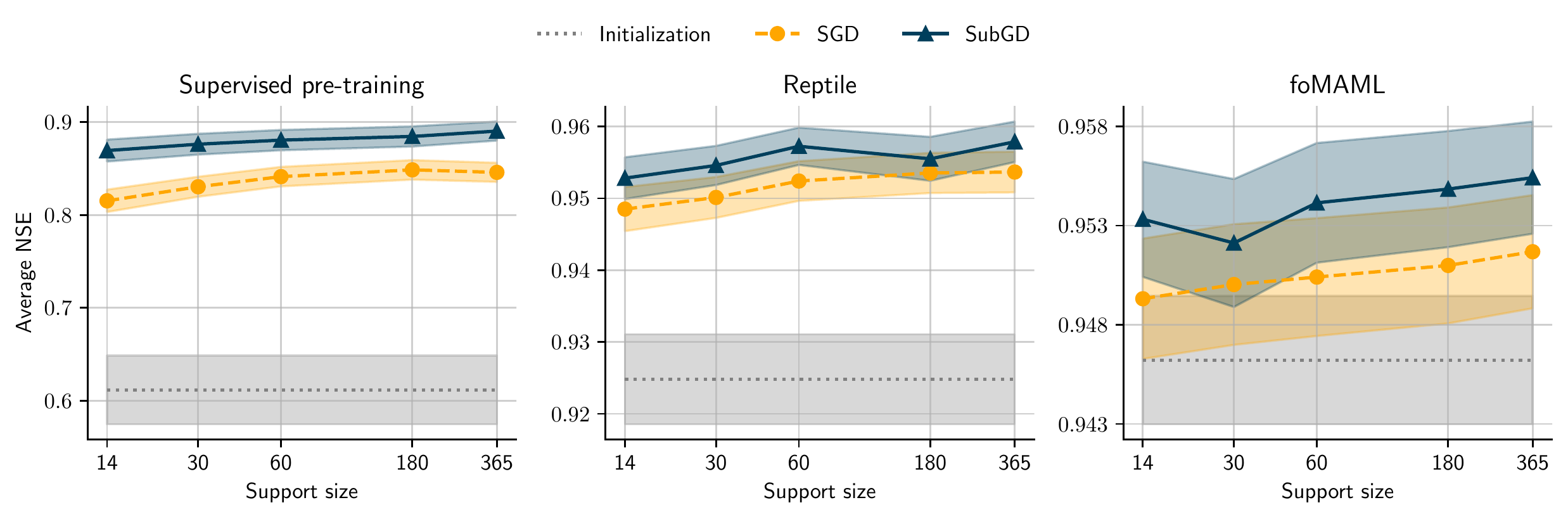}
    \caption{Average of the NSE metric (defined for $(-\infty, 1]$, where values closer to one are desirable) across 225 test tasks with increasing support size for SGD and SubGD optimization after different pre-training methods. Note the different limits of the y-axes. For all three initializations, SubGD outperforms SGD fine-tuning.}
    \label{fig:hydro-nse}
\end{figure}

\begin{table}[t]
\caption{NSE values on the hydrology dataset for increasing support size. We compare fine-tuning with SGD and SubGD after conventional supervised pre-training, MetaSGD, Meta-Curvature, foMAML, and Reptile training. Higher values are better, and bold values indicate methods that are not significantly worse than the best method ($\alpha=0.01$). For all support sizes, the best model employs SubGD.}
\label{tab:hydro-nse}
\begin{center}
\begin{tabularx}{\linewidth}{Xrrrrrrrrrr}
\toprule
\multirow{2}{*}{Method} & \multicolumn{2}{c}{14-shot} & \multicolumn{2}{c}{30-shot} & \multicolumn{2}{c}{60-shot} & \multicolumn{2}{c}{180-shot} & \multicolumn{2}{c}{365-shot} \\
 & mean & median & mean & median & mean & median & mean & median & mean & median \\
\midrule
MetaSGD & 0.946 & 0.971 & 0.946 & 0.971 & 0.948 & 0.971 & 0.951 & 0.972 & 0.952 & 0.972 \\
Meta-Curvature & 0.933 & 0.962 & 0.933 & 0.962 & 0.931 & 0.960 & 0.935 & 0.962 & 0.936 & 0.962 \\
Supervised+SGD & 0.815 & 0.875 & 0.830 & 0.889 & 0.841 & 0.896 & 0.849 & 0.905 & 0.846 & 0.900 \\
Supervised+SubGD & 0.869 & 0.935 & 0.876 & 0.934 & 0.881 & 0.939 & 0.885 & 0.941 & 0.890 & 0.938 \\
foMAML & 0.949 & 0.969 & 0.950 & 0.971 & 0.950 & 0.971 & 0.951 & 0.972 & 0.952 & 0.972\\
foMAML+SubGD & \bfseries 0.953 & \bfseries 0.972 & \bfseries 0.952 & \bfseries 0.972 & 0.954 & 0.972 & 0.955 & 0.973 & 0.955 & 0.972 \\
Reptile & 0.948 & 0.971 & 0.950 & 0.972 & 0.952 & 0.973 & 0.954 & 0.975 & 0.954 & 0.975 \\
Reptile+SubGD & \bfseries 0.953 & \bfseries 0.973 & \bfseries 0.955 & \bfseries 0.973 & \bfseries 0.957 & \bfseries \bfseries 0.976 & \bfseries 0.955 & \bfseries 0.977 & \bfseries 0.958 & \bfseries 0.977 \\
\bottomrule
\end{tabularx}
\end{center}
\end{table}

\subsection{Climate Change Adaption}
Climate change poses an increasing risk to human life \citep{ipcc2022climate}.
A globally warming climate has distinct effects on local weather phenomena, as it impacts inherent properties of the atmosphere \citep{trenberth2014global, shepherd2014atmospheric}.
These changes have major consequences on environmental modeling.
Existing local prediction systems will need adjustments to new behavior induced by climatic changes \citep{milly2008stationarity}.
Naturally, the amount of data available to adapt models is scarce at the onset of these changes.
Few-shot learning techniques are therefore key to the future success of environmental models based on deep learning which witness increasing adoption.
This experiment examines the potential use of SubGD for adaptation to new environmental conditions.
We consider rainfall--runoff simulation models.
Given meteorological data, these models predict river discharge, which is used for flood prediction.

A powerful mechanism to tackle environmental changes that were previously unseen at a given place is the paradigm of ``trading space for time'' \citep{singh2011trading, peel2011hydrological}.
Here, the idea is that future behavior may be new to one location, but not to other points in the world.
In other words, we assume that future behavior is similar to current or past behavior elsewhere.
We can therefore use ample meteorological data from other regions as substitutes for as of yet unknown possible future conditions at a location of interest.
The performance of few-shot learning approaches in this setting informs us about useful model adaptation strategies once local changes actually occur.

We use SubGD to adapt an LSTM rainfall--runoff model to the conditions from other locations. 
Starting from different pre-trained initializations (supervised pre-training, foMAML, Reptile, MetaSGD, and Meta-Curvature), we apply SGD and SubGD fine-tuning on support sets of varying size.
The detailed setup is outlined in Appendix~\ref{sec:appendix:hydrology}.

The results in Figure~\ref{fig:hydro-nse} and Table~\ref{tab:hydro-nse} show the Nash--Sutcliffe Efficiency (NSE) metric of MetaSGD and Meta-Curvature as well as SubGD and SGD after fine-tuning from initializations learned by Reptile, foMAML, and conventional supervised pre-training. 
The NSE is the most common evaluation criterion in hydrology and is defined as the $R^2$ between ground truth and predictions. 
Pre-training techniques that already consider the future goal of quick adaptation (foMAML, Reptile) ultimately yield better fine-tuned models.
SubGD is the most sample-efficient method and achieves significantly higher mean and median NSEs across all support sizes --- no matter which initialization we start with.
With the more sophisticated pre-training approaches (foMAML, Reptile), the differences become small, largely because both SGD and SubGD fine-tuning achieve good predictions with limited room for improvement.
Despite its comparably poor adaptation capabilities, supervised pre-training on data from the location itself is an important comparison, because it closely matches the machine learning setups that are currently deployed by practitioners.

\section{Conclusion and Outlook}
This contribution has introduced SubGD, a novel few-shot learning method that restricts gradient descent to a suitable subspace to increase sample efficiency.
We have built upon the fact that, on large datasets, SGD updates tend to live in a low-dimensional parameter subspace.
SubGD modifies update steps to leverage this property for few-shot learning in a way that is informed by training tasks.
Our update rule projects these steps into a low-dimensional subspace and rescales each direction.
An advantage of SubGD lies in its flexibility.
We can apply it on top of any initialization-based few-shot learning method or combine it with any gradient-based optimizer.
Furthermore, we have underpinned our method by a bound that connects the rank of the preconditioning matrix to the generalization abilities of SubGD. 

Our empirical investigations on problem settings from engineering and climatology show that SubGD is well-suited for dynamical systems applications.
In these settings, different tasks often stem from changes in the parameters of the underlying dynamical system.
Empirically, such changes translate into low-dimensional subspaces for neural networks. 

Future research may examine other approaches to determine subspaces, which might help to broaden the applicability of SubGD toward other few-shot learning settings where subspaces are harder to identify (e.g., image classification).
For instance, one could already incentivize low-dimensional, shared subspaces when fine-tuning on the training tasks.
Further, one could use methods such as autoencoders that extract non-linear manifolds instead of linear subspaces from the fine-tuning trajectories.

\section{Reproducibility}
Code and data to reproduce our experiments are available at \url{https://github.com/ml-jku/subgd}.

\subsubsection*{Acknowledgments}
The ELLIS Unit Linz, the LIT AI Lab, and the Institute for Machine Learning are supported by the Federal State Upper Austria. IARAI is supported by Here Technologies. We thank the projects AI-MOTION (LIT-2018-6-YOU-212), AI-SNN (LIT-2018-6-YOU-214), DeepFlood (LIT-2019-8-YOU-213), Medical Cognitive Computing Center (MC3), INCONTROL-RL (FFG-881064), PRIMAL (FFG-873979), S3AI (FFG-872172), DL for GranularFlow (FFG-871302), AIRI FG 9-N (FWF-36284, FWF-36235), and ELISE (H2020-ICT-2019-3 ID: 951847). We further thank Audi.JKU Deep Learning Center, TGW LOGISTICS GROUP GMBH, Silicon Austria Labs (SAL), FILL Gesellschaft mbH, Anyline GmbH, Google, ZF Friedrichshafen AG, Robert Bosch GmbH, UCB Biopharma SRL, Merck Healthcare KGaA, Verbund AG, Software Competence Center Hagenberg GmbH, T\"{U}V Austria, Frauscher Sensonic, and the NVIDIA Corporation.

\clearpage
\bibliography{collas2022_conference}
\bibliographystyle{collas2022_conference}

\appendix
\section{Appendix}
%You may include other additional sections here.

\subsection{Sinusoid Regression}
\label{app:sinus}

Table~\ref{tab:sinusoid-mse-onestep} lists the average MSE on the sinusoid dataset for SubGD and the reference methods after one update step.

\begin{table}[h]
\caption{MSE values on the sinusoid dataset for 5-, 10-, and 20-shot regression after one update step. SubGD and MAML start their fine-tuning procedure from the same initialization (50000 iterations of MAML training).}
\label{tab:sinusoid-mse-onestep}
\begin{center}
\begin{tabularx}{0.65\linewidth}{Xrrrrrr}
\toprule
\multirow{2}{*}{Method} & \multicolumn{2}{c}{5-shot} & \multicolumn{2}{c}{10-shot} & \multicolumn{2}{c}{20-shot} \\
 & mean & median & mean & median & mean & median \\
\midrule
MAML & 0.550 & 0.341 & 0.382 & 0.250 & 0.214 & 0.120 \\
Reptile & 0.989 & 0.777 & 0.548 & 0.378 & 0.342 & 0.238 \\
MetaSGD & 0.410 & 0.230 & 0.268 & 0.123 & 0.143 & 0.086 \\
MT-net & 0.554 & 0.330 & 0.371 & 0.264 & 0.166 & 0.094 \\
Meta-Curvature & 1.998 & 1.203 & 2.487 & 1.589 & 2.732 & 1.816 \\
SubGD & 0.807 & 0.358 & 0.780 & 0.567 & 0.625 & 0.474 \\
\bottomrule
\end{tabularx}
\end{center}
\end{table}

For MAML, MetaSGD, Meta-Curvature, and MT-net, we used the hyperparameters as reported in the respective papers.
For Reptile, we optimized hyperparameters ourselves.
Table~\ref{tab:sinus-hp} lists the final hyperparameters for each method.

\begin{table}[t]
\caption{Hyperparameters for MAML, Reptile, MetaSGD, MT-net, and Meta-Curvature on the sinusoid regression dataset. Support and query sizes 5/10/20 indicate that we separately trained each method for each support size.}
\label{tab:sinus-hp}
\begin{center}
\begin{tabularx}{\linewidth}{lXXXXXp{1.5cm}XXX}
\toprule
Method & Meta-batch size & Iterations & Optimizer & Inner-loop steps & Outer learning rate & Inner learning rate & Support size (train) & Query size (train) & Query size (val/test) \\
\midrule
MAML & 25 & 50000 & Adam & 1 & 0.001 & 0.01 & 5/10/20 & 10 & 100 \\
Reptile & 25 & 50000 & SGD & 10 & 1.0 & 0.005 (20-shot: 0.01) & 5/10/20 & 5/10/20 & 100 \\
MetaSGD & 25 & 50000 & Adam & 1 & 0.001 & 0.01 & 5/10/20 & 10 & 100 \\
MT-net & 4 & 60000 & Adam & 1 & 0.001 & 0.01 & 5/10/20 & 5/10/20 & 10 \\
Meta-Curvature & 25 & 70000 & Adam & 1 & 0.001 & 0.001 & 5/10/20 & 5/10/20 & 10 \\
\bottomrule
\end{tabularx}
\end{center}
\end{table}

\subsection{Non-linear RLC}
\label{sec:appendix:rlc}
In this section, we provide more details on the non-linear RLC system, dataset, and training procedure.

\subsubsection{Details on the RLC System and the Dataset}
\label{sec:appendix:rlc_system_dataset}
The RLC system can be described by the non-linear second-order differential equation
\begin{equation}
    \label{eq:rlc}
    \begin{pmatrix}
    \Dot{v}_C(t) \\ \Dot{i}_L(t)
    \end{pmatrix}
    = 
    \begin{pmatrix}
    0 & \frac{1}{C} \\
    - \frac{1}{L(i_L)} & - \frac{R}{L(i_L)}
    \end{pmatrix}
    \begin{pmatrix}
    v_C(t) \\ i_L(t)
    \end{pmatrix}
    + 
    \begin{pmatrix}
    0 \\ \frac{1}{L(i_L)}
    \end{pmatrix}
    v_{in}(t) , 
\end{equation}
where the inductance $L$ is modelled with a non-linear dependency on the inductor current $i_L$ and parameter $L_0$:
\begin{equation*}
    L = L(i_L) = L_0 \left[ 0.9 \left(\frac{1}{\pi} \mathrm{arctan} (-5 |i_L| - 5) + 0.5\right)+0.1\right].
\end{equation*}
Such a dependency typically arises in ferrite inductors that operate in partial saturation \citep{capua_nl_rlc_2017}.
In this system, the state $x(t)=\left[v_c(t), \; i_L(t)\right]^\top \in \mathbb{R}^2$ contains the voltage across the capacitor $v_C$ and the current through the inductor $i_L$. 
The output of this system $y(t) = v_C(t) \in \mathbb{R}$ is the capacitor voltage and the input $u(t) = v_{in}(t) \in \mathbb{R}$ is the input voltage to the electrical circuit denoted as $v_{in}(t)$.

We use the dynamical system defined by Equation~\eqref{eq:rlc} to generate our training and test tasks, where a task is defined as a specific instantiation of the system parameters resistance $R$ (unit $\Omega$), inductance $L_0$ (unit $\mathrm{\mu H}$), and capacitance $C$ (unit $\mathrm{nF}$). 
We generate 512 training tasks and 256 test tasks by sampling $R$, $L_0$ and $C$ uniformly from distributions with the ranges $[1, 14)$, $[20, 140)$, and $[100, 800)$, respectively. 
Different tasks correspond to different dynamical systems, which we want to approximate with a neural network after observing few samples.
For each task, we generate our ground truth data by simulating multiple state trajectories of each system.
The initial condition is always $x_0 = \left[0, \; 0\right]^\top$, but the input signal $v_{in}(t)$ is random filtered white noise with bandwidth $80\, \mathrm{kHz}$ and standard deviation $80 \mathrm{V}$. 
The resulting dataset for each task consists of three sequences that we discretize into 2000 time steps with a step size of $T_s = 1 \mathrm{\mu s}$. 
The ground truth simulation is performed with an explicit Runge--Kutta method of order 5(4). 
Following \citet{forgione_si_2022}, we also add Gaussian noise with a standard deviation of 0.1 to the outputs $y$ after simulation to represent measurement noise. An example of how the tasks without the measurement noise look like is depicted in Figure~\ref{fig:rlc-tasks-inputoutput}.

\begin{figure}
    \centering
    \includegraphics[width=0.8\linewidth]{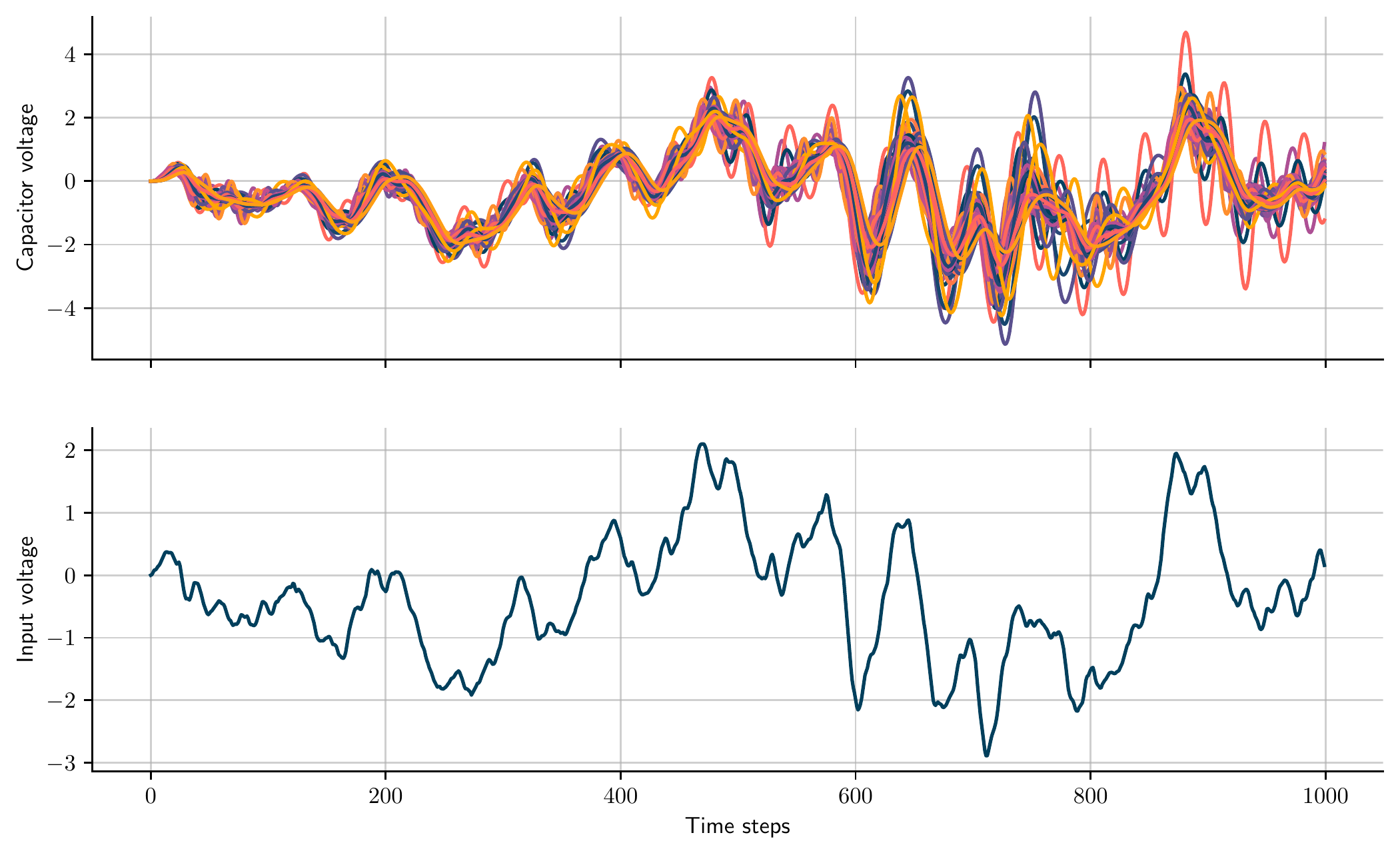}
    \caption{Output voltage of 50 different test systems simulated with the same input signal for visualization purposes. Shown for 1000 time steps.}
    \label{fig:rlc-tasks-inputoutput}
\end{figure}

\subsubsection{Details on the Training Objective}

To train a model on the data from the RLC system, we adopt a technique called ``truncated simulation error minimization'' introduced by \citet{forgione_si_2021}. In the following, we first review the main aspects of this method and then outline the changes we made.

The goal is to learn a neural network $f_\theta$ such that it approximates a dynamical system (e.g., given by Equation~\eqref{eq:rlc}), which yields the ordinary differential equation
\begin{equation}
\label{eq:neural_ss}
\begin{split}
    \Dot{\hat{x}} =& \, f_\theta(\hat{x}, u) \\
\end{split}
\end{equation}
with initial condition $x_0$, where $\hat{x}$ are the predicted system states, $\hat{y} = g(\hat{x})$ are the predicted outputs and $u$ are the inputs to the system. 
The output function $g$ can be any function, potentially with learnable parameters, but in our case it is just a linear mapping without any additional parameters.
To solve this differential equation which is parameterized by the neural network we can use a numerical solution scheme of choice: 
\begin{equation}
    \hat{x}(t; \theta, x_0) = \mathrm{ODEINT}\left(t, f_\theta(\cdot, \cdot), u(\cdot), x_0\right).
\end{equation}
This yields a state trajectory from which the predicted output $\hat{y}$ can be computed for every time step using the output function $g$. 

\citet{forgione_si_2021} argue that naively minimizing the mean squared error between predicted outputs $\hat{y}$ and ground truth $y$ on full dataset trajectories might become computationally too expensive because of the large computation graphs. 
Instead, they sample ``truncated'' subsequences of fixed length from the dataset, which can be processed in batches. 

Additionally, they introduce a set of free, learnable variables $\Bar{X} = \{\Bar{x}_0, \Bar{x}_1, \dots, \Bar{x}_N\}$ with the same dimension as the state space, where $N$ corresponds to the number of data points in the dataset. 
These variables are meant to capture the subsequences' initial conditions. 
As such, they should become consistent with the learned function $f_{\theta}$.
Consequently, they take advantage of the additional information by training the free variables jointly by minimizing the dual loss 
\begin{equation}
    \cL_{\mathrm{tot}}(\theta, \Bar{X}) = \cL_{\mathrm{fit}}(\theta, \Bar{X}) + \alpha \cL_{\mathrm{reg}}(\theta, \Bar{X}), 
\end{equation}
where $\alpha \geq 0$ is a regularization weight, $\cL_{\mathrm{fit}}$ is the MSE objective on the outputs, and $\cL_{\mathrm{reg}}$ is the consistency loss, which penalizes the distance between the predicted state $\hat{x}_t$ and the free variable $\Bar{x}_t$ for each time step. 
For further details, we refer to \citet{forgione_si_2021}.

\citet{forgione_si_2021} initialized the variables $\bar{X}$ with noisy measurements of all state variables. 
We cannot adopt this convention since we assume that not all state variables are observable (and thus not contained in the dataset). 
\citet{forgione_si_2022} initialize the hidden states in $\Bar{X}$ to zero. 
We found that with this initialization the consistency loss $\cL_{\mathrm{reg}}$ in combination with the additional parameters $\Bar{X}$ does not improve the final performance on output prediction 
and that the system is not able to capture the hidden dynamics. 
For this reason, we only use the fit loss term $\cL_{\mathrm{fit}}$ 

\subsubsection{Details on the model and the training procedures}
As described in Section~\ref{introduction}, SubGD has three training stages:\ (1) pre-training, (2) fine-tuning, and (3) evaluation.
In the following, we provide details on the neural network models and each of these phases for the RLC dataset.

The pre-training stage performed on the training tasks yields a model which we adapt to the test tasks with SubGD using only few data samples.
SubGD does not require any specific pre-training procedure, which enables the combination with other few-shot learning methods. 
In this experiment, we apply SubGD update steps on models pre-trained with supervised training, MAML, or Reptile.
During fine-tuning, the subspace for SubGD is determined on update directions obtained by fine-tuning the pre-trained model to each training task separately. 
Finally, we use SubGD to adapt our model to the test tasks and evaluate its performance under the constraint of limited data.

For all training procedures, we parameterize $f_\theta$ as a feed-forward neural network with one hidden layer, where the number of inputs corresponds to the state and input dimension of the system.
The hidden layer has 50 units and is followed by a tanh nonlinearity, and the two output units correspond to the state variables.
As numerical solution scheme for the resulting neural ODE, we use forward Euler with a constant step size that is equal to the dataset's discretization interval $T_s$. 
Since the output $\hat{y}$ is contained in the state, the function $g(\hat{x}) = \left[1, \, 0 \right]\hat{x}$ is a simple linear mapping with no additional parameters. This setup is adopted from \citet{forgione_si_2022}.

\paragraph{Supervised Pre-training} 
For the supervised pre-training, we generate another (nominal) task in the same way as the training and test tasks (see Section~\ref{sec:appendix:rlc_system_dataset}) with parameters $R = 3 \, \Omega$, $C = 270 \, \mathrm{nF}$, $L_0 = 50 \, \mathrm{\mu H}$ and train our model with learning rate $0.0001$ for 30 epochs with mini-batches containing 16 (truncated) sequences of length 256 that are sampled from random positions in the dataset. 
As initial conditions for the Euler-discretized neural state space model (see Equation~\eqref{eq:neural_ss}), we use the observable state variable (i.e., $v_C$) one time step before each batch sequence starts and set the unobservable states (i.e., $i_L$) to zero. We use the Adam optimizer for this training phase \citep{kingma_2015_adam}.

\paragraph{foMAML and Reptile Pre-training}
We use the 512 training tasks (see Section~\ref{sec:appendix:rlc_system_dataset}) to train our model with foMAML and Reptile. 
This requires to sample support and query sets for each task, where we define these sets to be non-overlapping sequences of equal length (i.e., a support or query size of 50 time steps) starting at random positions in the dataset.
The corresponding initial values for each sequence are determined in the same way as in supervised pre-training. 
All additional hyperparameters are given in Table~\ref{tab:rlc-hp}.

\paragraph{Fine-tuning on Training Tasks}
In order to calculate the preconditioning matrix that determines the subspace of SubGD, we need to collect update directions on the training tasks.
This is done by fine-tuning the pre-trained model for each training task separately in the same way as in supervised pre-training, but with a smaller learning rate ($5e^{-6}$) and with early stopping ($5$ epochs) to determine the number of fine-tuning epochs.

\paragraph{Evaluation}
The final evaluation of SubGD is performed on test tasks (see Section~\ref{sec:appendix:rlc_system_dataset}) with varying support sizes (see Figure~\ref{fig:rlc-support-vs-mse}), where the support size corresponds to the length of the sequence used for few-shot model adaptation to the task at hand.
For the RLC dataset, this sequence is taken from the beginning of a state trajectory in the dataset.  
Hence, model adaptation is always performed with the true initial conditions used for data generation. Note that there is no batching as in the previous training stages. This means we only use a single sequence (with length equal to the respective support size) for model adaptation.
Finally, the metric (i.e., the MSE) is calculated on the full 2000 time steps of another trajectory.

\begin{table}[t]
\caption{Hyperparameters for MAML, Reptile, MetaSGD, and Meta-Curvature on the RLC dataset.}
\label{tab:rlc-hp}
\begin{center}
\begin{tabularx}{\linewidth}{rXXXXX>{\raggedright}p{2.2cm}XX}
\toprule
Method & Meta-batch size & Iterations & Optimizer & Inner-loop steps & Outer learning rate & Inner learning rate & Support size (train) & Query size (train) \\
\midrule
foMAML & 16 & 50000 & Adam & 5 & 0.001 & 0.001 & 50 & 50 \\
Reptile & 10 & 50000 & SGD & 5 & 0.001 & 0.001 & 50 & 50 \\
Meta-Curvature & 6 & 50000 & Adam & 5 & 0.001 & 0.0001 & 50 & 50 \\
MetaSGD & 6 & 50000 & Adam & 1 & 0.001 & randomly from [1e-4, 5e-4] & 50 & 50 \\
\bottomrule
\end{tabularx}
\end{center}
\end{table}

\subsection{Climate Change Adaption}
\label{sec:appendix:hydrology}

We refer to \citet{Gauch2020AML} for a general introduction of the rainfall--runoff prediction setting that is geared toward machine learning practitioners.

\subsubsection{Setup Details}

Our experimental setup for the climate change adaptation experiment is as follows:\ we use meteorological inputs from the DayMet weather observation data product \citep{thornton2012daymet} and set the prediction targets with the \hbv-EDU model \citep{seibert2012teaching, aghakouchak2010application}. From these targets, we create our observations by adding Gaussian noise with a standard deviation of 0.2 in the normalized space.
We adopt the sequence-to-one prediction setting from \cite{kratzert_toward_2019} and congruently use an LSTM as the machine learning model, which constitutes the state-of-the-art approach for runoff simulation \citep{kratzert_toward_2019,kratzert_regional_2019}. We choose a hidden size of $20$ and an input sequence length of $365$ days \citep{kratzert2019neuralhydrology}.
The LSTM uses 4 meteorological inputs (precipitation, temperature, solar radiation, and vapor pressure). 

During training, we use output dropout in the hidden layer after the LSTM with a rate of 0.2.
We train with an MSE loss function and stop training after 1000 training iterations without validation improvement.
We tune hyperparameters at support size 60 with three different random seeds and choose the configuration with the best average NSE.
Table~\ref{tab:hydro-hp} lists the hyperparameters of the different baseline methods.
For supervised pre-training, we use a mini-batch size of 64, a learning rate of 0.005, and apply early stopping after the validation NSE did not improve for 6 epochs.

We compare SubGD fine-tuning against the baselines of first-order MAML (foMAML), MetaSGD, Meta-Curvature, and Reptile.
For the initial training phase, we use 400 basins with 9 years of training and 8 years of validation data (1989-10-01 -- 1999-09-30 and 1980-10-01 -- 1989-09-30, respectively).
For SubGD, we fine-tune on the same time period of each of these 400 basins individually to collect training trajectories from which we derive the SubGD subspace.
This ensures that SubGD does not have access to more data than the baseline methods.
We then randomly choose 50 of these basins to search for a suitable learning rate and number of update steps of SGD and SubGD fine-tuning on few samples. 
As in the sinusoid and RLC experiments, we apply SubGD to the initializations learned by the different pre-training methods (supervised, foMAML, and Reptile).
Finally, following the paradigm of trading space for time, we evaluate all methods on data from 225 previously unseen basins (i.e., locations that were not part of the 400 used for training) on the period 1990-10-01 -- 1999-09-30. 
We randomly draw support samples from the first year and use the subsequent years as the query set. 
 
\begin{figure}
    \centering
    \includegraphics[width=\linewidth]{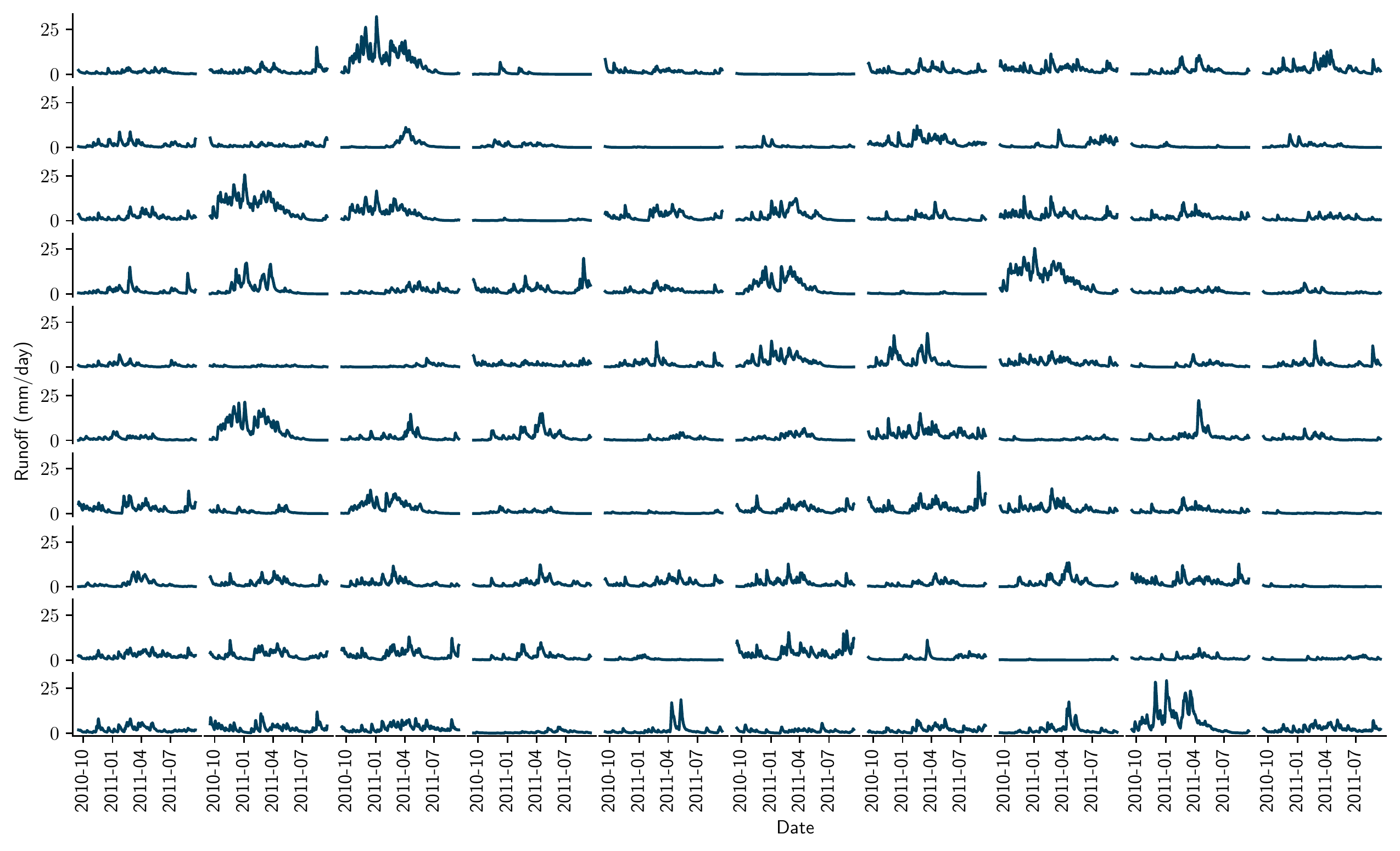}
    \caption{Examples for runoff target time series of 100 basins for the hydrological year 2010-10-01 to 2011-09-30.}
    \label{fig:hydro-qsim}
\end{figure}

\begin{table}[t]
\caption{Hyperparameters for MAML, Reptile, MetaSGD, and Meta-Curvature on the climate change adaptation dataset.}
\label{tab:hydro-hp}
\begin{center}
\begin{tabularx}{\linewidth}{rXXXXX>{\raggedright}p{2.2cm}XX}
\toprule
Method & Meta-batch size & Iterations & Optimizer & Inner-loop steps & Outer learning rate & Inner learning rate & Support size (train) & Query size (train) \\
\midrule
foMAML & 6 & 50000 & Adam & 5 & 0.01 & 0.001 & 60 & 60 \\
Reptile & 6 & 50000 & SGD & 10 & 1.0 & 0.1 & 60 & 60 \\
Meta-Curvature & 6 & 50000 & Adam & 1 & 0.01 & 1e-5 & 60 & 60 \\
MetaSGD & 6 & 50000 & Adam & 1 & 0.01 & randomly from [5e-5, 1e-4] & 60 & 60 \\
\bottomrule
\end{tabularx}
\end{center}
\end{table}

\subsubsection{Details on the \texorpdfstring{\hbv-EDU}{HBV-EDU} Model}
Models are a central part of hydrological inquiry and practice. They are used extensively within flood forecasting systems and water resources management, but are also central to assess the impact of potential future changes in climate and land use.

The \hbv~model is the prototypical conceptual rainfall--runoff model. It was originally conceptualized by the Swedish Meteorological and Hydrological Institute to predict the inflow of hydropower plants \citep{bergstrom1992hbv, lindstrom1997development}. Today, variations of the model with varying degrees of complexity are widely used in research and practice. An extensive retrospective of the development can be found in \citet{seibert2021retrospective}. 

\hbv-EDU is one of these variations.
Originally devised for presentation and educational purposes, its ease of use and clarity led to its wide adoption as a reference model.
The model accounts for five different processes to represent the rainfall--runoff relationship: (1) effective precipitation, (2) snow melt and accumulation, (3) evapotranspiration, (4) soil moisture, and (5) runoff generation. To model this chain, it requires meteorological inputs in the form of precipitation and temperature at each time step, the day of the year to adjust for the angle of incoming radiation, and monthly means of potential evapotranspiration and temperature to provide seasonal adjustments for the effective evapotranspiration. All in all, \hbv-EDU has 11 parameters that can adjust the internal process representation. 

\subsection{MiniImageNet Experiments}
\label{app:mini}
For our experiments on miniImageNet, we use the four-layer CNN architecture as proposed in \citet{vinyals_matching_2016}.
We use the 64 classes of the training split as a  64-way classification pre-training task.
For this pre-training, we optimize for 200 epochs using the Adam optimizer \citep{kingma_2015_adam} with a learning rate of $10^{-3}$, an early stopping patience of 10, and a batch size of 256.
Starting from the resulting model, we individually fine-tune on 1024 5-way 5-shot tasks that are randomly drawn from the same training classes. Here, one model is trained for each task separately. 
The difference of the resulting fine-tuned parameters and the pre-training parameters yields the 1024 vectors that we use to calculate the effective rank shown in Figure~\ref{fig:miniimmage-erank}.

Next, we determine the number of update steps and the learning rate for fine-tuning on new tasks in a grid search on 20 5-way 5-shot tasks from the validation split. 
The aforementioned 1024 vectors are used to estimate the matrix $C$ that is used in SubGD (see Equation~\eqref{eqn:subgd_update}).
Finally, we compare the performance of fine-tuning with Adam and SubGD on a randomly selected set of 100 5-way 5-shot tasks from the test split.
Using Adam for fine-tuning results in an accuracy of 0.558, while SubGD achieves an accuracy of 0.539.
In a two-sided Wilcoxon signed-rank test, the performance difference is assigned a $p$-value of $6.65^{-05}$. Thus, we find SubGD to be unsuited for miniImageNet.
This finding is in line with the absence of significant flattening of the effective rank (see Figure~\ref{fig:miniimmage-erank}), which indicates that the models corresponding to miniImageNet tasks do not span a low-dimensional parameter subspace.

\section{Details on the Generalization Bound}
\label{sec:generalization_bound_proof}

In this Section, we present further details for the generalization bound~\eqref{eq:complexity_of_est_error} that motivates our few-shot learning approach.
In particular, we follow the notation of~\citet{long_generalization_2020} (Subsection~\ref{subsec:notation}) to introduce a new class of models (Subsection~\ref{subsec:class_of_models}) which are derivable by our algorithm and we provide a new generalization bound for these models (Subsection~\ref{subsec:gen_bound}). Finally (Subsection~\ref{subsec:extensions}) we briefly discuss how possible extensions (e.g., to more general network architectures) may look like.

\subsection{Notation}
\label{subsec:notation}
Let $\mathcal{X}$ be a real-valued compact \textit{input} space, $\mathcal{Y}\subset\mathbb{R}$ be a \textit{target} space and $P$ be a distribution (Borel probability measure) on $\mathcal{X}\times\mathcal{Y}$.
Given some \textit{sample} (multiset) $S=\{(x_1,y_1),\ldots,(x_m,y_m)\}$ independently drawn from $P$, the goal is to find a function $f:\mathcal{X}\to\mathcal{Y}$ such that the \textit{risk (generalization error)} $\mathbb{E}_{z\sim P}[\ell_f(z)]$ is small.
Here, $\ell:\mathcal{Y}\times\mathcal{Y}\to[0,1]$
denotes some \textit{loss} function and $\ell_f(z)=\ell(f(x),y)$ for $z=(x,y)$.
In this work, we focus on loss functions which are $l$-Lipschitz in the first argument for some $l\geq 1$.

For simplicity and to not overload notation, we follow~\citet{long_generalization_2020} and consider the class $F$ of convolutional neural networks with the following properties (a discussion on possible extensions can be found in Subsection~\ref{subsec:extensions}):
\begin{itemize}
    \item All layers use zero padding.
    \item The activation functions are $1$-Lipschitz and nonexpansive, e.g., ReLU and tanh.
    \item The kernels of the convolution layers are $K^{(i)}\in\mathbb{R}^{b\times b\times c\times c}$ for all $i\in\{1,\ldots,L\}$.
    \item The weight vector $w$ of the last (linear) layer is fixed such that $\norm{w}_2=1$.
\end{itemize}
We denote by $n=L b^2 c^2$ the total number of trainable parameters.
The parameter vector $\theta\in \mathbb{R}^n$ of some network $f_\theta\in F$ is the composition of all elements in the kernels $K^{(1)},\ldots,K^{(L)}$ of $f$.
We denote by $\mathrm{op}(\theta^{(i)})$ the operator matrix of the kernel $K^{(i)}$ (the matrix representing the linear kernel function),
by $\norm{x}_1$ the $\ell_1$-norm of a vector $x$, 
and by $\norm{M}_2$ the operator norm of a matrix $M$ w.r.t.~the Euclidean vector norm, i.e., the spectral norm.
If the loss function $\ell$ and the neural network architecture $f$ is clear from the context, we denote by
$\mathcal{L}(\theta):=\mathbb{E}_{z\sim P}[\ell_{f_\theta}(z)]$ the risk
and by $\mathcal{L}_S(\theta):=\frac{1}{|S|}\sum_{z\in S} \ell_{f_\theta}(z)$ the empirical risk.
We follow~\citet{long_generalization_2020} and define the norm $\norm{\theta}_\sigma:=\sum_{i=1}^L \|\mathrm{op}(\theta^{(i)})\|_2$.

By~\citet{sedghi2018singular,long_generalization_2020} it holds that $\norm{\theta}_\sigma\leq\norm{\theta}_1$, and thus $\norm{\cdot}_{\sigma}$ can always be replaced by $\norm{\cdot}_1$ in the subsequent calculations, without any further effort. Our main Theorem~\ref{thm:generalization_bound} then immediately yields the bound~\eqref{eq:complexity_of_est_error}, which we formulated in terms of the $1$-norm, in order to not introduce too much new notation in the main text. 

\subsection{Class of models coming from SubGD}
\label{subsec:class_of_models}

In the following, we define the class of models which can be obtained by SubGD as follows:
\begin{definitionA}
\begin{align*}
F_\mathrm{SubGD}=\left\{f_\theta\in F\mid \theta=\theta_0+\hat C\cdot\theta',\theta'\in\mathbb{R}^n\right\},
\end{align*}
where $\theta_0\in\mathbb{R}^{n}$ is the initialization of SGD.
\end{definitionA}
%as defined in Eq.~\eqref{eq:sgd}.

\subsection{Generalization bound for SubGD}
\label{subsec:gen_bound}

\begin{theoremA}[Generalization Bound]
\label{thm:generalization_bound}
Let $f_{\theta_0}\in F$ be a convolutional neural network with kernels normalized such that $\|\mathrm{op}({\theta_0^{(i)}})\|_2=1$. 
Then, there is an absolute constant $\zeta>0$ such that the following holds with probability at least $1-\delta$, for all $\delta>0$ over the choice of a sample $S$ of size $m$, for all models $f_\theta\in F_\mathrm{SubGD}$ with $\norm{\theta-\theta_0}_\sigma\leq \beta$:

If $\beta\geq 2$ then
\begin{align}
    \mathcal{L}(\theta)\leq \mathcal{L}_S(\theta)
    +\zeta\sqrt{\frac{\mathrm{rank}(\hat C)\left(\beta+\log(l)\right)+\log(1/\delta)}{m}}
\end{align}
and otherwise
\begin{align}
    \mathcal{L}(\theta)\leq \mathcal{L}_S(\theta)
    +\zeta\left(\beta l\sqrt{\frac{\mathrm{rank}(\hat C)}{m}}+\sqrt{\frac{\log(1/\delta)}{m}}\right).
\end{align}
\end{theoremA}
Theorem~\ref{thm:generalization_bound} studies the distance of a SubGD-based model to a model $f_{\theta_0}\in F$ with normalized weights.
Note that the normalization assumption is for simplicity only and can be extended to more general parameters in a straight forward way, see e.g.\ Theorem~3.1 in~\citet{long_generalization_2020}.

\subsection{Proofs}
\label{sec:proof_of_bound}

Theorem~\ref{thm:generalization_bound} can be proven by showing that the class
\begin{align*}
\ell_{F_\mathrm{SubGD}^\beta}:=\left\{\ell_{f_\theta}\mid f_{\theta}\in F_\mathrm{SubGD},\left\lVert \theta-\theta_0\right\rVert_\sigma\leq\beta\right\}
\end{align*}
of loss functions applied to SubGD-based models $F_{\mathrm{SubGD}}$, is $\left(\beta l e^\beta,\mathrm{rank}(\hat C)\right)$-Lipschitz parametrized.
\begin{definitionA}[Lipschitz parametrized]
\label{def:lipschitz}
A class $G$ of functions on a common real-valued domain $Z$ is $(B,d)$-Lipschitz parametrized for some $B,d>0$, if there is a norm $\norm{.}$ on $\mathbb{R}^d$ and a function $\phi$ from the unit ball in $\mathbb{R}^d$ (w.r.t.~the norm) mapping to $G$ such that
\begin{align*}
    \left|(\phi(\theta))(z)-(\phi(\theta'))(z)\right|\leq B \left\lVert\theta-\theta'\right\rVert
\end{align*}
for all $z\in Z$ and all $\theta,\theta'$ with $\norm{\theta}\leq 1$ and $\norm{\theta'}\leq 1$.
\end{definitionA}
Theorem~\ref{thm:generalization_bound} follows directly from classical arguments in statistical learning theory, see e.g., Lemma 2.3 in~\citet{long_generalization_2020}, if we show Definition~\ref{def:lipschitz} for our class of loss functions applied to SubGD-based models.
That is, we need to prove the following lemma.

\begin{lemma}
\label{lemma:lipschitz}
The function class $\ell_{F_\mathrm{SubGD}^\beta}$
is $(\beta l e^\beta,\mathrm{rank}(\hat C))$-Lipschitz parametrized.
\end{lemma}
The following proof partially relies on techniques developed in~\citet{long_generalization_2020}.
\begin{proof}[Proof of Lemma~\ref{lemma:lipschitz}]
Our goal is to show that there exists a norm $\left\lVert.\right\rVert$ on $\mathbb{R}^{\mathrm{rank}(\hat C)}$ and a mapping $\phi$ from the unit ball of $\mathbb{R}^{\mathrm{rank}(\hat C)}$ to $\ell_{F_\mathrm{SubGD}^\beta}$ such that
\begin{align*}
|\phi(\omega)(z)-\phi(\omega')(z)|\leq \beta l \, e^\beta \left\lVert \omega-\omega'\right\rVert
\end{align*}
holds for all $\omega,\omega'\in\mathbb{R}^{\mathrm{rank}(\hat C)}$ with $\left\lVert\omega\right\rVert\leq 1$ and $\left\lVert\omega'\right\rVert\leq 1$.

We start our proof by the observation that
\begin{align*}
    F_\mathrm{SubGD}=\left\{f_\theta\in F\mid \theta=\theta_0+V\,\omega,\,\omega\in\mathbb{R}^{\mathrm{rank}(\hat C)}\right\}
\end{align*}
with some $V\in\mathbb{R}^{n\times \mathrm{rank}(\hat C)}$ which contains as columns a basis of the vector space spanned by the columns of $\hat C$. The existence of such a $V$ follows from the existence and unique dimension of the Hamel basis. $V$ can also be chosen to be the matrix introduced in Section~\ref{methods}.

Let us define $\norm{\omega}_\mathrm{SubGD}:=\norm{V \omega}_\sigma$ for $\omega\in\mathbb{R}^{\mathrm{rank}(\hat C)}$ and note that $\norm{.}_\mathrm{SubGD}$ is a norm since $\norm{\omega}_\mathrm{SubGD}=0\implies \omega={\bf 0}$ as $\norm{.}_\sigma$ is a norm and $V$ has full column rank, $\norm{\alpha\, \omega}_\mathrm{SubGD}=|\alpha| \norm{\omega}_\mathrm{SubGD}$ for $\alpha\in\mathbb{R}$, and, $\norm{\omega+\omega'}_\mathrm{SubGD}\leq \norm{\omega}_\mathrm{SubGD}+\norm{\omega'}_\mathrm{SubGD}$ for all $\omega,\omega'\in\mathbb{R}^{\mathrm{rank}(\hat C)}$.

For some $\omega,\omega'\in\mathbb{R}^{\mathrm{rank}(\hat C)}$ such that $\norm{\omega-\omega_0}_\mathrm{SubGD}\leq \beta$ and $\norm{\omega'-\omega_0}_\mathrm{SubGD}\leq \beta$ with $\norm{\mathrm{op}((\theta_1+V\omega_0)^{(j)})}_2=1$, Lemma~2.6 in~\citet{long_generalization_2020} gives
\begin{align}
\nonumber
    \left|\ell(f_{\theta_1+V\omega}(x),y)-\ell(f_{\theta_1+V\omega'}(x),y)\right|&\leq l\, e^\beta \left\lVert V\omega-V\omega'\right\rVert_\sigma\\
    \label{eq:lipschitz_ineq}
    &= l\, e^\beta \left\lVert \omega-\omega'\right\rVert_\mathrm{SubGD}.
\end{align}

Consider now the mapping $\phi_\mathrm{SubGD}$ on the unit ball in $\mathbb{R}^{\mathrm{rank}(\hat C)}$ w.r.t.~the norm $\norm{.}_\mathrm{SubGD}$ defined by
\begin{align*}
    \phi_\mathrm{SubGD}(\omega)(x,y)=\ell(f_{\theta_1+V (\omega_0+\beta \omega)}(x),y).
\end{align*}
If $\widetilde\omega$ and $\widetilde\omega'$ are from the unit ball, then $\norm{\widetilde{\omega}}_\mathrm{SubGD}\leq 1$ and $\norm{\widetilde{\omega}'}_\mathrm{SubGD}\leq 1$.
Consequently, $\norm{(\omega_0+\beta \widetilde{\omega})-\omega_0}_\mathrm{SubGD}\leq \beta$ and $\norm{(\omega_0+\beta \widetilde{\omega}')-\omega_0}_\mathrm{SubGD}\leq \beta$ and the function maps to the class $\ell_{F_\mathrm{SubGD}^\beta}$.
By applying Equation~\eqref{eq:lipschitz_ineq} with $ \omega:=\omega_0+\beta \widetilde{\omega}$ and $ \omega':=\omega_0+\beta\widetilde{\omega}'$ we get
\begin{align*}
|\phi_\mathrm{SubGD}(\omega)(z)-\phi_\mathrm{SubGD}( \omega')(z)| &= |\ell(f_{\theta_1+V \omega}(x),y)-\ell(f_{\theta_1+V \omega'}(x),y)|\\
&= |\ell(f_{\theta_1+V (\omega_0+\beta \widetilde{\omega})}(x),y)-\ell(f_{\theta_1+V (\omega_0+\beta \widetilde{\omega}')}(x),y)|\\
&\leq \beta\, l\, e^\beta \norm{ \widetilde{\omega}-\widetilde{\omega}'}_\mathrm{SubGD}.
\end{align*}
\end{proof}

\begin{proof}[Proof of Theorem~\ref{thm:generalization_bound}]
Theorem~\ref{thm:generalization_bound} directly follows from Lemma 2.3 in~\citet{long_generalization_2020} by noting that $\beta > \log(\beta)$, $l\geq 1$ and $\beta\geq 2\implies \beta l e^\beta\geq 5$ for $\beta>0$.
\end{proof}

Let us end this subsection by providing an argument for an upper bound on $\norm{\theta_{k}-\theta_{0}}_1$ that should bring into play the effect of the matrix $\hat C$ on the iteration. Specifically, let us show that:
\begin{equation} \label{eq:norm_bound}
\norm{\theta-\theta_0}_1\leq \sqrt{n}\, k\,\eta\,l\!\sum_{i=1}^{n}\! \sigma_i.
\end{equation}
Thus the lower the dimension of the subspace our method focuses on, the more eigenvalues are zero.
\begin{proof}[Proof of norm bound \eqref{eq:norm_bound}]
We have
\begin{align*}
    \norm{\theta_{k}-\theta_{0}}_1 &\leq \sum_{t=1}^k \norm{\theta_{t}-\theta_{t-1}}_1\\
    &= \sum_{t=1}^k \norm{\eta\, \hat C\, \nabla \ell(f_{\theta_{t-1}}(x),y)}_1\\
    &\leq \sum_{t=1}^k \sqrt{n}\, \eta\, \norm{\hat C \, \nabla \ell(f_{\theta_{t-1}}(x),y)}_2\\
    &\leq \sum_{t=1}^k \sqrt{n}\, \eta\, \norm{\hat C}_2 \norm{\nabla \ell(f_{\theta_{t-1}}(x),y)}_2.
\end{align*}
Since the loss function $\ell$ is $l$-Lipschitz we have $\norm{\nabla \ell(f_{\theta_t}(x),y)}_2\leq l$ for all $t$ and consequently
\begin{align*}
    \norm{\theta_{k}-\theta_{0}}_1 &\leq k \,\eta\, l\, \sqrt{n} \norm{\hat C}_2 \leq k\, \eta\, l\, \sqrt{n} \sum_{i=1}^{n}\!\sigma_i
\end{align*}
since the operator norm is smaller than or equal to the nuclear norm.
\end{proof}

\subsection{Possible extensions and future work}
\label{subsec:extensions}
Our setting introduced in Section~\ref{subsec:notation}, immediately extends to the following scenarios:
\begin{itemize}
    \item Fully connected layers: instead of using $\mathrm{op}(\theta^{(i)})$ (the matrix corresponding to the convolution operation), one can use any other matrix as well and bound its operator norm analogously.
    \item Multivariate outputs: $\mathcal{Y}\subset\mathbb{R}^p$.
    \item Different loss functions: $\ell$ maps into $[0,M]$.
    \item Different normalization: $\|\mathrm{op}({\theta_0^{(i)}})\|_2\le 1+\nu$.
\end{itemize}
Incorporating all these modifications is straightforward by introducing further notation. In this case, the bounds from Theorem~\ref{thm:generalization_bound} only change slightly. For detailed proof arguments for these extensions, we refer the interested reader to Section~3 in  \citet{long_generalization_2020}. We further believe that our result can be extended to recurrent networks as well, using the arguments and assumptions from \citet{chen2019generalization_rnn} and we leave the details open for possible future work.

\end{document}